\documentclass[runningheads]{llncs}
\usepackage{graphicx}
\usepackage{amsmath}
\usepackage{url}
\usepackage{microtype}
\usepackage{subfig}
\usepackage{booktabs} 
\usepackage{dsfont}
\usepackage{amssymb}
\usepackage{algorithm}
\usepackage{algorithmic}
\usepackage{forloop}
\usepackage{comment}
\usepackage[numbers]{natbib}
\usepackage{float}
\usepackage{tabularx}
\usepackage[inline]{enumitem}
\usepackage{mathtools}
\usepackage{dsfont}
\usepackage{tikz-cd}
\usepackage{mathabx}
\usepackage{esvect}
\usepackage{tikz}

\newcommand\mycitet[1]{\cite{#1}}

\newcommand{\id}{\mathcal{I}}

\newcolumntype{C}[1]{>{\centering\arraybackslash}p{#1}}
\newlist{inlinelist}{enumerate*}{1}
\setlist*[inlinelist,1]{%
	label=(\roman*),
}

\begin{document}
\title{Adversarial Computation of Optimal Transport Maps}
%
%
\author{Jacob Leygonie$^\star$\inst{1,2} \and
Jennifer She$^\star$\inst{2,3}\and
Amjad Almahairi\thanks{Authors contributed equally.}\inst{1} \and \\
Sai Rajeswar\inst{1,2} \and
Aaron Courville\inst{2}}
\authorrunning{J. Leygonie et al.}
%
\institute{Element AI, Canada 
\and MILA, Universit{\'e} de Montr{\'e}al, Canada
\and Stanford University, USA
}
\maketitle              
\begin{abstract}
Computing optimal transport maps between high-dimensional and continuous distributions is a challenging problem in optimal transport (OT). Generative adversarial networks (GANs) are powerful generative models which have been successfully applied to learn maps across high-dimensional domains. However, little is known about the nature of the map learned with a GAN objective. To address this problem, we propose a generative adversarial model in which the discriminator's objective is the $2$-Wasserstein metric. We show that during training, our generator follows the $W_2$-geodesic between the initial and the target distributions. As a consequence, it reproduces an optimal map at the end of training. We validate our approach empirically in both low-dimensional and high-dimensional continuous settings, and show that it outperforms prior methods on image data.

\end{abstract}

\section{Introduction}
\label{introduction}

The computation of optimal maps is a fundamental problem in optimal transport (OT), with various applications to image processing~\citep{gramfort2015fast}, 
computer graphics~\citep{solomon2015convolutional}, and
transfer learning~\citep{courty2017optimal}.
Informally, an optimal map (or Monge map) transforms one probability measure into another, while minimizing ``cost''. Formally, given two probability measures $\mu$ and $\nu$ on the Euclidean space $\mathds{R}^m$, and a cost function $c :\mathds{R}^m\times \mathds{R}^m\rightarrow \mathds{R}$, an optimal map $T : \mathds{R}^m \to \mathds{R}^m$ is any function achieving the infimum of the \emph{Monge problem},
\begin{equation} \label{MONGE_PROBLEM}
 \inf_{T \in \mathcal{A}} \int_{\mathds{R}^m} c(x,T(x)) d\mu(x),
\end{equation}
where $\mathcal{A}$ is the set of all functions from $\mathds{R}^m$ to $\mathds{R}^m$ such that $\nu = T_\#\mu$ is the \emph{push-forward} of $\mu$ through $T$, that is, $\nu(A) = T_\#\mu(A) = \mu(T^{-1}(A))$ for all measurable sets $A \subseteq \mathds{R}^m$.

Computing optimal maps can be quite challenging, especially for high-dimensional and continuous spaces~\citep{Peyre2018}.
Most recent approaches find approximate optimal maps by solving a regularized relaxation of the Monge problem. 
Notably, Cuturi~\mycitet{cuturi2013sinkhorn} proposes the Sinkhorn algorithm which can only handle discrete and semi-discrete distributions, and its time-complexity is quadratic with respect to the input space. Genevay~\emph{et al.}~\mycitet{Genevay2016} propose a stochastic gradient method which can be applied to continuous distributions. However, the size of their model grows linearly with the number of input samples, which does not scale to high-dimensional input measures.
More recently, Seguy~\emph{et al.}~\mycitet{Seguy2018LSOT} learn a neural-network-based optimal map which can handle probability measures in high-dimensional spaces more efficiently. However, for high-dimensional image data, their approach is not competitive with state-of-the-art generative models, such as generative adversarial networks (GANs)~\citep{goodfellow2014generative}.


GANs learn a parametric map (the generator network) between two arbitrary distributions through an adversarial game with another discriminator network. In generative modelling applications, the generator maps a low-dimensional latent space to a high-dimensional target space. Recently, GANs were successfully applied in learning maps between two domains of arbitrary dimension in an unsupervised fashion~\citep{zhu2017unpaired, almahairi2018augmented}.
Motivated by the remarkable success of GANs in large-scale image domains, we devise a GAN-based method for learning optimal maps, which can be applied to continuous and high-dimensional spaces.


With existing GAN frameworks, it is challenging to ensure that the generator's learned map is truly optimal in the OT sense. Characterizing the generated map relies critically on the choice of the discriminator's objective. On the one hand, objectives based on $f$-divergence~\citep{goodfellow2014generative,NIPS2016_6066} bear no clear connection with OT theory. On the other hand,
objectives based on the $1$-Wasserstein distance~\citep{arjovsky2017wasserstein} are strongly related to the Monge problem~\eqref{MONGE_PROBLEM}. Nonetheless, the $1$-Wasserstein metric lacks properties such as the uniqueness of both the optimal map and the $1$-Wasserstein geodesics~\citep{bottou2018geometrical}, which are important for concluding that the learned map is optimal. 

In this paper, we make the following contributions:
\begin{itemize} [noitemsep,topsep=0pt]
\item We introduce W2GAN, in which the discriminator computes the $2$-Wasserstein metric.
\item We study the path of evolution of the generated distribution in W2GAN during training. Under ideal assumptions (convergence of the discriminator and an infinite-capacity generator) we show that the generated distribution evolves along the unique $2$-Wasserstein geodesic between the initial distribution and the target distribution. To the best of our knowledge, characterizing the evolution of a GAN's generator is an unsolved problem.
\item As a result of this analysis, we show that the generator recovers an optimal map at convergence.
\item We show that our analysis can be extended to practical cases, where the ideal assumptions do not hold, by bounding the deviation of the trajectory of generated distributions from the ideal $2$-Wasserstein geodesic.
\item We verify the theoretical properties of our model on synthetic low-dimensional data, and show that W2GAN performs competitively in learning optimal maps for high-dimensional image data.
\end{itemize}


\section{Related Work}

Generative adversarial networks~\cite{goodfellow2014generative} are powerful probabilistic generative models which have attained state-of-the-art results, especially for high-dimensional image data. OT has contributed to the empirical success of GANs, mainly by motivating more robust training objectives. Arjovsky~\emph{et al.}~\mycitet{arjovsky2017wasserstein} proposed using $1$-Wasserstein distance as an alternative to $f$-divergences~\citep{NIPS2016_6066}.
The objective of the discriminator in WGAN and its extensions~\citep{gulrajani2017improved,Miyato2018spectral} corresponds to a particular instance of a relaxation of the Monge problem \eqref{MONGE_PROBLEM}. 
Salimans~\emph{et al.}~\mycitet{Salimans2017otgan} proposed replacing the discriminator with the Sinkhorn algorithm, which can compute any Wasserstein metric. Alternatively, Sanjabi~\emph{et al.}~\mycitet{Sanjabi2018swgan} used the optimal transport map as a new way to train a GAN discriminator. Similarly, Genevay~\emph{et al.}~\mycitet{genevay2017learning} train a GAN model using a Wasserstein loss whose cost is optimized in an adversarial way. In this work, we propose a contribution in the other direction: we propose a GAN-based model to solve a fundamental problem in OT, namely computing optimal transport maps between distributions.

There has been considerable
recent interest in using GANs for unsupervised learning of maps across domains~\cite{zhu2017unpaired,almahairi2018augmented}. Most of these methods rely on heuristics and architectural features to ensure the learning of meaningful maps. 
Recently, however, Lu~\emph{et al.}~\mycitet{lu2018guiding} use an estimated OT map as a reference to guide the map learned by a GAN toward satisfying some task-specific properties. Our work shares a similar motivation, but instead modifies the GAN objective in order to compute the optimal map between its initial and target distributions.



On the theoretical side, OT provides tools for analyzing the training dynamics of the generator. Lei~\emph{et al.}~\mycitet{Lei2017AGV} characterize the update step of a generator in Wasserstein GANs, and suggest that a GAN minimizing the $2$-Wasserstein distance would enjoy desirable properties. More global descriptions of the generator's evolution throughout training are proposed by Bottou~\emph{et al.}~\mycitet{bottou2018geometrical}. In particular, Wasserstein GANs are prone to follow $W_1$-geodesics, which are not unique. On the contrary, the $W_2$-geodesics are unique~\citep{Villani,ambrosio2008gradient}. This is, in fact, a crucial ingredient for characterizing the training dynamics of our $W_2$-based GAN, as we discuss in section~\ref{sec:gen_recover_ot}.

\section{Background}
In this section, we state some basic notions from OT theory~\citep{Villani, Ambrosio2013} which will be necessary for the rest of the paper. A more comprehensive presentation is provided in Appendix~\ref{OT_THEORY_BACKGROUND}. \\

\noindent \textbf{Optimal transport problem and Wasserstein distances} Computing Wasserstein distances and optimal transport maps are two intrinsically related problems. 
When the point-wise cost function $c$ in the Monge problem~\eqref{MONGE_PROBLEM} is the Euclidean distance raised to the power of a non-negative integer $p$, and the measures $\mu,\nu$ are absolutely continuous with respect to the Lebesgue measure, the solution to \eqref{MONGE_PROBLEM} corresponds to the $p$-th Wasserstein distance, denoted $W_p(\mu,\nu)$.\footnote{Rigorously, optimizing~\eqref{MONGE_PROBLEM} with cost function $c(x,y):=\frac{1}{p}\|x-y\|^p$ (where $\|.\|$ denotes the Euclidean norm) leads to $W_p^p(\mu,\nu)$.}
We will assume this form for $c$ in the rest of the paper. In this case, it can be shown that we recover the value of $W_p(\mu, \nu)$ with a relaxed formulation of \eqref{MONGE_PROBLEM}, which is called the dual formulation of the \textit{Kantorovich problem}:
\begin{gather}
\underset{\phi,\psi \in \mathcal{A}^* }{\text{sup}}
 \int_{ x \in \mathcal{X}} \phi(x) d\mu(x) + \int_{ y \in \mathcal{Y}} \psi(y) d\nu(y), \nonumber\\
\mathcal{A}^*:= 
\{(\phi,\psi): \mathds{R}^d\rightarrow \mathds{R} \mid
\forall x,y  \in \mathcal{X}\times \mathcal{Y},
 \phi(x)+\psi(y) \leqslant c(x,y)\}.
   \label{KANTOROVITCH_DUAL}
\end{gather}

Each coordinate of a pair ($\phi,\psi$) maximizing~\eqref{KANTOROVITCH_DUAL} is called a \textit{Kantorovitch potential}. We will refer to the constraint in the definition of $\mathcal{A}^*$ as the \emph{c-inequality constraint} or simply the \emph{inequality constraint}. \\


\noindent \textbf{Relating Potentials to the Monge map}
An important property of $W_2$ is that we can relate the dual Kantorovitch solution $\phi$ to the Monge map $T$~\citep{brenier1991polar}. Specifically, $T$ is unique and is determined by:
\begin{equation}
  T= \id- \nabla \phi,
\label{T_GRAD_PHI}
\end{equation}
where $\id$ is the identity. A detailed and more generalized statement of this relationship is given in Proposition~\eqref{prop:uniqueRelation} in Appendix~\ref{OT_THEORY_BACKGROUND}. \\

\noindent \textbf{Geodesics in the space of probability measures}
\label{sec:geodesics_W_2}
Recall that $W_2$ is a metric over the space of absolutely continuous distributions on $\mathds{R}^d$~\citep{Villani}. A \textit{constant speed $W_2$-geodesic} between two distributions $\mu$ and  $\nu$ is a set $\mathcal{G}(\mu,\nu)=\{\mu_t\}_{0\leqslant t \leqslant 1}$ of distributions such that $\mu_0=\mu$, $\mu_1=\nu$, and
\begin{equation}
   \forall \; 0\leqslant s \leqslant t\leqslant1, \; W_2(\mu_s,\mu_t)=(t-s)W_2(\mu,\nu). 
\end{equation}
In the $W_2$ case, such geodesics are unique~\citep{Villani}. 
Moreover, we can explicitly relate geodesics and Monge maps. Given the unique Monge map $T$ between $\mu$ and $\nu$, the constant speed geodesic is given by $\mu_t:=T_t\#\mu$ where $T_t=(1-t)\id+tT$. 


\section{W2GAN: Model Description}
In this section, we introduce the W2GAN model in which the discriminator computes the $2$-Wasserstein distance.
Let $P_x$ be the target distribution in $\mathds{R}^m$ and $P_z$ be the input distribution in $\mathds{R}^d$. Here $P_z$ does not necessarily refer to a low-dimensional probability measure of standard GANs. In our experiments, we consider the case $d=m$ and aim at finding an optimal map between $P_z$ and $P_x$. The model is composed of two components: a generator $G:\mathds{R}^d\rightarrow \mathds{R}^m$, and a discriminator $(\phi,\psi)$ which is instantiated by two real parametric functions defined on $\mathds{R}^m$.

For optimization purposes, the discriminator computes the $2$-Wasserstein distance using a regularized version of~\eqref{KANTOROVITCH_DUAL}. We choose to use  $L_2$ regularization, but one could use other penalties~\citep{cuturi2013sinkhorn,Seguy2018LSOT,Blondel2018}.\footnote{The same regularized approach is taken for the gradient penalty or Lipschitzness of the discriminator of WGAN~\citep{gulrajani2017improved, petzka2018on}. This Lipschitzness constraint is in fact a particular case of the general $c$-inequality constraint in the case of $W_1$~\citep{Villani}. Hence, WGAN-GP relies on the exact same procedure: softening the hard constraint by adding a penalty in the discriminator's objective.} The W2GAN discriminator's objective for computing $W_2^2(G\#P_z,P_x)$ is given by:
\begin{equation}
 \begin{gathered}
       \sup_{\phi,\psi} \mathds{E}_{(z,x)\sim P_z \times P_x}\left[\phi(G(z))+\psi(x) 
        -\lambda \mathcal{L}_{\text{ineq}}(\phi,\psi,x,z) \right], \\
    \mathcal{L}_{\text{ineq}}(\phi,\psi,x,z) \coloneqq \left(\phi(G(z))+\psi(x)-\frac{\|G(z)-x\|_2^2}{2}\right)_{+}^2,
\end{gathered}
      \label{LOSS_DISCRIMINATOR_GAN_INEQONLY}
 \end{equation}
 where $(.)_+:= \max(0,.)$ and $\lambda$ is a scalar controlling the strength of the penalty $\mathcal{L}_{\text{ineq}}$.
 W2GAN generator's objective is defined as: 
  \begin{equation}
       \inf_G \mathds{E}_{z\sim P_z}(\phi(G(z)).
      \label{LOSS_GENERATOR_W2}
 \end{equation}
The training procedure of W2GAN is detailed in Algorithm~\ref{alg:example1}. In Appendix~\ref{sec:objective_details} and~\ref{sec:parametrization_details}, we provide some possible extensions to W2GAN's training objective and parametrization.

\begin{algorithm}[t]
  \caption{W2GAN training method.}
  \label{alg:example1}
\begin{algorithmic}
  \REQUIRE Inequality constraint strength $\lambda>0$, Number of critic iterations $n_\mathrm{critic}$. 
  \REQUIRE Initial parameters $w_0$ of $\phi$, $v_0$ of $\psi$ and $\theta_0$ of $G$.
  \WHILE{$\theta$ has not converged}
  \FOR{$t=1, ..., n_{\mathrm{critic}}$}
  \STATE Sample mini-batch $\{x^{(i)}\}_{i=1}^{m} \sim P_x$ and $\{z^{(i)}\}_{i=1}^{m} \sim P_z$.
  \STATE $\mathcal{L}_{D} \leftarrow \frac{1}{m} \sum_{i=1}^{m} \left[\phi(G(z^{(i)}))+\psi(x^{(i)}) -\lambda \mathcal{L}_{\text{ineq}}(\phi,\psi,x^{(i)}, z^{(i)}) \right]$
  \STATE $w \leftarrow w + \nabla_w \mathcal{L}_D$
  \STATE $v \leftarrow v + \nabla_v \mathcal{L}_D$
  \ENDFOR
  \STATE Sample mini-batch $\{z^{(i)}\}_{i=1}^{m} \sim P_z$.
  \STATE $\theta \leftarrow \theta + \nabla_\theta \frac{1}{m} \sum_{i=1}^{m}\phi(G(z^{(i)}))$.
  \ENDWHILE
\end{algorithmic}
\end{algorithm}

\section{Recovering a Monge Map with W2GAN's Generator}
\label{sec:gen_recover_ot}
In this section, we show that the W2GAN's generator recovers a Monge map at the end of training. Our theoretical analysis relies on characterizing the evolution of the generated distribution throughout training. We develop our analysis in three sub-sections:
\begin{enumerate} 
    \item First, we introduce a parameter-free update rule for the generator, which allows for a more direct connection with theoretical OT notions than the usual parametric update rule. 
    \item Then we consider the generator's evolution under some strong but informative assumptions. Namely, at each training iteration, we assume that the discriminator is optimal and that the generator's parametric update rule is equivalent to the parameter-free update rule. With these assumptions in place, we show that during training the generator follows the unique $W_2$ geodesic between the initial and the target distributions, and as a result that it recovers an optimal map at convergence.
    \item Finally, we show that when we drop the assumptions above, the generated distribution does not deviate much from the ideal trajectory.
\end{enumerate}

Let us introduce some notations required for our analysis. 
Let $G_\theta:\mathds{R}^d\rightarrow \mathds{R}^m$ be the generator function specified by the choice of a parameter $\theta$. Thus, at a each training iteration $t\in \mathds{N}$, we denote the generator corresponding to the current parameters $\theta_t$ as $G_{\theta_t}$. We denote the discriminator at time $t$ by $(\phi_t,\psi_t)$. Since the generator's loss \eqref{LOSS_GENERATOR_W2} depends only on $\phi_t$, we will simply refer to $\phi_t$ as the discriminator.
Subsequent generator parameters $\theta_t, \theta_{t+1}$ are related through the stochastic gradient descent update rule as follows: 
\begin{align}
\theta_{t+1}= \theta_t- \alpha E_{z\sim P_z} \nabla_{\theta_t} (\phi_t(G_{\theta_t}(z))),
\label{update_rule_param}
\end{align}
where $\alpha>0$ is the learning rate. Our goal is to characterize the following:
\begin{definition}
\normalfont
A \textit{training sequence} of the generator is a sequence $(G_{\theta_0},G_{\theta_1},...)$ of generator functions specified by parameters $(\theta_0,\theta_1,...)$ and is \textit{induced} by the sequence of discriminators $(\phi_0, \phi_1,...)$. For any $t\in \mathds{N}$, $\theta_{t+1}$ is obtained from $\theta_t$ and $\phi_t$ using the update rule \eqref{update_rule_param}. 
\end{definition}
To disambiguate between random variables and their probability measures, we define $\mu_\theta:=G_\theta \# P_z$, i.e the probability measure of $G_\theta(z)$ where $z\sim P_z$. When the context is clear, we shall also refer to the generated probability measures $(\mu_{\theta_0},\mu_{\theta_1},...)$ as the training sequence of the generator.  

We refer to the process that produces $G_{\theta_{t+1}}$ from $G_{\theta_t}$ using \eqref{update_rule_param} as the \textit{parametric update rule}. In order to connect the training sequence with OT notions, dealing directly with this update rule is not convenient as it limits the family of realizable generators. In the next paragraph, we explain how to device a parameter-free generator function $G_{t+1}$ from $G_{\theta_{t}}$ via a \textit{functional update rule}, which is a parameter-free version of the parametric update rule. The function $G_{t+1}$ should be thought of as $G_{\theta_{t+1}}$ in the limit of infinite parametrization.

\subsection{Defining the functional update rule} We want to have a notion of steepest gradient descent in a space of functions rather than in the parameter space. This is the purpose of the \textit{functional update rule} that produces a generator $G_{t+1}$ given a generator $G_{\theta_t}$ and a discriminator $\phi_t$. We derive this update rule using functional gradient analysis, similar to~\mycitet{yamaguchi2018distributional,johnson2018composite}. Specifically, the update rule emerges from taking the derivative of the following loss in the function space: 
$\mathcal{L}'(H):=\mathcal{L}(\phi_t,H\circ G_{\theta_t})$ where $H:\mathds{R}^m \rightarrow \mathds{R}^m$. Observe that $\mathcal{L}'(\id)$ is simply the loss $\mathcal{L}(\phi_t,G_{\theta_t})$ of the generator, where $\id$ is the identity on $\mathds{R}^m$. Intuitively, $\mathcal{L}'(H)$ is the loss given by the discriminator to the probability measure $H\# \mu_{\theta_t}$, which is a translation of the current generated distribution by $H$. In order to find a direction that decreases the discriminator's loss $\mathcal{L}'$, we need a notion of derivative of $\mathcal{L}'$ at $\id$. We appeal to Gateaux derivatives: given a function $h:\mathds{R}^m \rightarrow \mathds{R}^m$, we consider the quantity\footnote{The last equality comes from the dominated convergence theorem. See \cite{yamaguchi2018distributional} for more details.}
\begin{align}
\label{eq:directional_derivative}
    \lim_{\epsilon \rightarrow 0}\frac{\mathcal{L}'(\id+ \epsilon h)-\mathcal{L}'(\id)}{\epsilon}&=\lim_{\epsilon \rightarrow 0}\mathds{E}_{z\sim P_z}(\frac{1}{\epsilon}[\phi(G_{\theta_t}(z)+\epsilon h(G_{\theta_t}(z)))-\phi_t(G_{\theta_t}(z)])\nonumber \\
&=\mathds{E}_{z\sim P_z}[\nabla \phi_t(G_{\theta_t}(z)) h(G_{\theta_t}(z))].
\end{align}
Equation \eqref{eq:directional_derivative} shows that the direction $h$ that maximally decreases the loss $\mathcal{L'}$ is $-\nabla \phi_t$. Hence we define the functional update rule as follows.
\begin{definition}
\normalfont
\label{derivative_functional_gradien_loss}
 Given a learning rate $\alpha'>0$, the \textit{functional update rule} consists in choosing the generator $G_{t+1}$ to be \[G_{t+1}=G_{\theta_t}-\alpha' \nabla \phi_t \circ G_{\theta_t}.\] 
\end{definition}
The functional and parametric update rules are strongly related \cite{yamaguchi2018distributional}. The parametric update rule, given a discriminator $\phi_t$ and parameter $\theta_t$ for the generator, consists in finding $\theta_{t+1}$ so that the generated distribution $G_{\theta_{t+1}}$ will be as close as possible to $G_{t+1}$. More precisely, $G_{\theta_{t+1}}$ is the result of a parametric update rule applied on a loss that forces the generator to be close to $G_{t+1}$. As a matter of fact, fix $\alpha'=1$ and consider the loss $\min_\theta \mathds{E}_{z\sim P_z}\|G_{t+1}(z)- G_{\theta}(z)\|_2^2$. If we device the update rule at $\theta=\theta_t$ using stochastic gradient descent with this loss, we recover the parametric update rule \eqref{update_rule_param}. Because this loss and $\mathcal{L}(\phi_t,G_{\theta_t})$ provide the same update rule, the resulting parametric generator $G_{\theta_{t+1}}$ is close to $G_{t+1}$ according to the $L_2(P_z)$ norm. 

\subsection{The generator's training sequence in the ideal case}
To obtain a strong connection between the theoretical notions of OT theory and the actual training of the generator, we first assume that the generator function is not limited by the parameter space and that the discriminator is optimal at each training iteration. These assumptions are typically used when studying the convergence of GANs~\cite{kodali2017convergence,mescheder2018training,lin2018pacgan}.

Formally, we consider the following set of assumptions:
\begin{inlinelist}
    \item All generated distributions $\mu_{\theta_t}$ and the target distribution are absolutely continuous measures with respect to the Lebesgue measure on $\mathds{R}^m$.
    \item The discriminator is \textit{perfect}: at each time $t>0$ the discriminator $\phi_t$ is a Kantorovitch potential for the Kantorovitch problem \eqref{KANTOROVITCH_DUAL} between $\mu_{\theta_t}$ and $P_x$.
    \item The updates are \textit{ideal}: the parametric update is the same as the functional update taken with $\alpha'=\alpha$, i.e for all $t>0$, $G_{\theta_{t+1}}=G_{t+1}$.
\end{inlinelist}

The first assumption is technically needed for the optimal transport map to be always uniquely defined. It may fail for the target distribution which can have support on a low-dimensional sub-manifold of $\mathds{R}^m$ and for the generated ones whose supports are manifolds of dimension bounded by the dimension $d$ of the input space.\footnote{This hypothesis is also standard in convergence analysis of GAN methods\citep{nagarajan2017gradient}. For instance, absolute continuity is implicitly used in any GAN method whose discriminator relies on a KL divergence, for such a quantity not to be well-defined for non absolutely continuous measures} On the other hand, we can always replace the distributions at hand by arbitrarily close absolutely continuous distributions, hence making this assumption more of a theoretical concern than a practical one. The other two assumptions need a more careful treatment. In practice, the discriminator is not perfect as it would mean it minimizes the loss \eqref{LOSS_DISCRIMINATOR_GAN_INEQONLY} by ignoring the effect of the inequality constraint. The assumption that the updates of the generator are ideal is also a strong because the parametrization limits the space of realizable functions. However, it remains informative to observe what happens under these hypotheses. 

With these assumptions in place, the following proposition shows that an update of the generated distribution $\mu_{\theta_t}$ ``follows" the optimal transport map joining $\mu_{\theta_t}$ and $P_x$.

\begin{proposition}
\normalfont
Given the generated distribution $\mu_{\theta_t}$ at time $t\in \mathds{N}$, the updated distribution $\mu_{\theta_{t+1}}$ lies on the $2$-Wasserstein geodesic joining $\mu_{\theta_t}$ and $P_x$, i.e $\mu_{\theta_{t+1}}\in \mathcal{G}(\mu_{\theta_t},P_x)$. Moreover, $G_{\theta_{t+1}}=(1-\alpha)G_{\theta_{t}}+ \alpha T_t \circ G_{\theta_{t}}$ where $T_t$ is the optimal transport map from $\mu_{\theta_t}$ to $P_x$.\footnote{We provide the proof of all propositions in Appendix~\ref{sec:proofs_main_text}.}
\label{local_update_0}
\end{proposition}
Proposition \ref{local_update_0} describes the local evolution of the generator. Together with the uniqueness and description of $W_2$-geodesics detailed in section \ref{sec:geodesics_W_2}, we can characterize the global evolution of the generator.
\begin{proposition}
\normalfont
\label{global_evolution_generated_distrib}
The training sequence is a subset of the $W_2$-geodesic between the initial distribution $\mu_{\theta_0}$ and the target distribution $P_x$, i.e. $\{\mu_{\theta_0},...,\mu_{\theta_t},...\} \subseteq \mathcal{G}(\mu_{\theta_0},P_x)$.
\end{proposition}
Furthermore, we can guarantee convergence of the training sequence toward the target distribution.
\begin{proposition}
\normalfont
\label{decreasing_ditance}
The distance $W_2(\mu_{\theta_t},P_x)$ decreases exponentially fast. That is,
\[\forall t \in \mathds{N}, \ W_2(\mu_{\theta_t},P_x)\leqslant (1-\alpha)^{t}W_2(\mu_{\theta_0},P_x)\]
\end{proposition}
As a results of the previous propositions, we show that the generator recovers the Monge map $T$ between $\mu_{\theta_0}$ and $P_x$.
\begin{proposition}
\normalfont
Denote $T$ the optimal map between $\mu_{\theta_0}$ and $P_x$. There is a decreasing function $f:\mathds{N}\rightarrow [0,1]$ with $f(0)=1$ such that
\label{generator_recovers_T}
\[\forall t \in \mathds{N}, \ \mu_{\theta_t}=(f(t)\id+ (1-f(t))T)\# \mu_{\theta_0} \]
\[\forall t \in \mathds{N}, \ G_{\theta_t}=(f(t)\id+ (1-f(t))T)\circ G_{\theta_0} \]
Moreover, $\mu_{\theta_t}$ (resp. $G_{\theta_t}$) converges exponentially fast toward $T \# \mu_{\theta_0}$ (resp. $T\circ G_{\theta_0}$) in the sense that $f(t)\leqslant (1-\alpha)^t$.
\end{proposition}

 The consequence of the proposition above is our desired result: the generator recovers the Monge map $T$ between the initial and target distributions, i.e $G_\infty=T\circ G_{\theta_0}$,  where $G_\infty$ is the generator in the limit of an infinite number of training iterations.
\begin{remark}
\normalfont
If the discriminator $(\phi,\psi)$ were perfect, we could simply use it together with the relationship \eqref{T_GRAD_PHI}, i.e $T(x)=x-\nabla \phi(x)$, to approximate a Monge map. Although this method performs reasonably well in low dimension (see Figure \ref{fig:local_OT}), due to the non-ideal instantiation of the discriminator as a neural network, it does not perform well in high-dimensional experiments. On the other hand, the generator in W2GAN empirically efficiently recovers a Monge map as shown in section \ref{sec:exps}. We believe this is because errors of the discriminators at each update of the generator are cancelling each other. Moreover, the relationship \eqref{T_GRAD_PHI} is more likely to hold in practice near the end of training, when the generated and true distributions are sufficiently close. In the next sub-section, we describe how much an imperfect discriminator perturbs the generator's trajectory.
\end{remark}

\subsection{Deviation from the ideal trajectory}
In this section, we analyze the consequences of dropping the assumptions of perfect discriminators and ideal updates of the generator. Specifically, we upper-bound the deviation of the generated distribution from the $W_2$-geodesic $\mathcal{G}(\mu_{\theta_0},P_x)$ with respect to the distance between the sub-optimal discriminator (resp. generator's parametric update) and the perfect discriminator (resp. functional update).\footnote{We analyze further the training sequence in the parameter space in Appendix \ref{parametric analysis}.} We first bound the error induced by a sub-optimal discriminator.
\begin{proposition}
\normalfont
\label{unperfect_discriminator}
Given the generated distribution $\mu_{\theta_t}$, consider a discriminator $\phi$ and the potential $\Tilde{\phi}$ solving the Kantorovitch problem \eqref{KANTOROVITCH_DUAL} between $\mu_{\theta_t}$ and $P_x$,\footnote{Notice that $\nabla \Tilde{\phi}$ is uniquely defined by Eqn. \ref{T_GRAD_PHI} and uniqueness of the Monge map} such that $\|\nabla \phi - \nabla \Tilde{\phi}\|_{\infty}\leqslant \epsilon$. Let $\mu_{\theta_{t+1}}$ and $\Tilde{\mu}_{\theta_{t+1}}$ be two ideally updated distributions, according to discriminators $\phi$ and $\tilde{\phi}$ respectively. Then $W_2(\mu_{\theta_{t+1}},\Tilde{\mu}_{\theta_{t+1}})\leqslant\frac{\alpha \epsilon}{\sqrt{2}}$.  
\end{proposition}
This result states that an error of $\epsilon$ in the approximation of the gradient of a Kantorovitch potential can result in a deviation of at most $\frac{\alpha \epsilon}{\sqrt{2}}$ of the generated distribution from the geodesic $\mathcal{G}(\mu_{\theta_t},P_x)$.

Next, we examine the additional deviation from the ideal trajectory that is induced by a non-ideal update of the generator. Let $\|.\|_2$ be the norm given by the inner product on $L^2(P_z)$.
\begin{proposition}
\normalfont
\label{deviation_ideal_update}
For a given a time step $t\in \mathds{N}$, let $\mu_{\theta_{t+1}}$ and $\Tilde{\mu}_{\theta_{t+1}}$ be the generated distributions obtained from $G_{\theta_t},\phi_t$ using the parametric and functional update rules, respectively. If $\|G_{\theta_{t+1}}-G_{t+1}\|_2\leqslant \epsilon'$ or $\|G_{\theta_{t+1}}-G_{t+1}\|_\infty\leqslant \epsilon'$ for some $ \epsilon' \geqslant 0$, then $W_2(\mu_{\theta_{t+1}},\Tilde{\mu}_{\theta_{t+1}})\leqslant \frac{\epsilon'}{\sqrt{2}}$. 
\end{proposition}
Finally, the total error of approximation can be bounded as follows.
\begin{corollary}
\normalfont
\label{total deviation}
If
\begin{inlinelist}
    \item the discriminator $\phi$ satisfies $\|\nabla \phi - \nabla \Tilde{\phi}\|_{\infty}\leqslant \epsilon$, where $\Tilde{\phi}$ is defined as in Proposition $5$, and 
    \item the parametric and functional updates $G_{\theta_{t+1}}$ and $G_{t+1}$ with respect to $\phi$ satisfy $\|G_{\theta_{t+1}}-G_{t+1}\|_2\leqslant \epsilon'$ or $\|G_{\theta_{t+1}}-G_{t+1}\|_\infty\leqslant \epsilon'$. 
\end{inlinelist}
then $W_2(\mu_{\theta_{t+1}},\Tilde{\mu}_{\theta_{t+1}})\leqslant \frac{\alpha\epsilon +\epsilon'}{\sqrt{2}}$ where $\Tilde{\mu}_{\theta_{t+1}}$ is the ideally updated distribution with respect to the perfect discriminator $\Tilde{\phi}$. 
\end{corollary}
This corollary states that we can bound the deviation of the training sequence of the generated distribution from the geodesic $\mathcal{G}(\mu_{\theta_t},P_x)$ when controlling the bias induced by the parametric update and the non-optimality of the discriminator. 

Since the W2GAN model is closely related to WGANs, a natural question is whether we can adapt the previous analysis to the case of discriminators optimizing $W_1$ instead of $W_2$.
Unfortunately, unlike $W_2$-geodesics, the $W_1$-geodesics between a fixed pair of absolutely continuous measures are non-unique. In the Appendix \ref{sec:W1_case}, we provide more details on how the difference between $W_1$ and $W_2$ makes it difficult to adapt our argument.

\section{Experiments}
\label{sec:exps}

\subsection{2-Dimensional Synthetic Data}

We first consider learning optimal maps between synthetic datasets in 2-dimensional space.\footnote{Detailed experimental setup is provided in Appendix~\ref{sec:exp_details} and code is provided in \url{https://github.com/jshe/wasserstein-2}.} Working with 2D data allows visualizing learned maps easily. Each
dataset is composed of samples from two distributions: $P_z$ and $P_x$, and we seek to learn optimal maps from $P_z$ to $P_x$. 
three datasets (samples are shown in Figure~\ref{fig:2d_data}-(a)). The three datasets we consider are:
\begin{inlinelist}
\item \textbf{4-Gaussians}: $P_z$ and $P_x$ are mixtures of 4 Gaussians with equal mixture weights, where the mixture centers of $P_x$ are closer to each other than those of $P_z$.
\item \textbf{Checkerboard}: $P_z$ and $P_x$ are mixtures of uniform distributions over 2D squares, of size 4 and 5 respectively, with equal mixture weights. The mixture centers of the two distributions form an alternating checkerboard pattern.
\item \textbf{2-Spirals}: $P_z$ and $P_x$ are uniform distributions over spirals that are rotations of each other.
\end{inlinelist}
We show in Figure~\ref{fig:2d_data}-(b) the optimal map between \emph{data samples}, obtained by applying the optimal assignment algorithm between samples. We refer to this method as \textit{Discrete-OT}, and use it as a reference for evaluating optimal maps obtained by various methods on these specific data samples.

We apply our proposed W2GAN model on these three datasets, where the generator takes input samples from $P_z$ and and maps them to $P_x$. Note that we first initialize the generator as an identity function.\footnote{We experiment with two methods for identity initialization: reconstruction and adding a skip connection to output with small initial weights.} 
As baselines, we compare with the following methods:
\begin{inlinelist}
\item \textit{Barycentric-OT}: a two-step algorithm for computing optimal maps between continuous or discrete distributions introduced by~\mycitet{Seguy2018LSOT}. The algorithm is based on first computing a regularized optimal transport
plan, then estimating a parametric optimal map as a barycentric projection.
\item WGAN-GP~\citep{gulrajani2017improved}, and
\item WGAN-LP~\citep{petzka2018on}. Both WGAN-GP and WGAN-LP use the $W_1$ metric as the discriminator's adversarial objective, but differ in the form of gradient penalty on the discriminator. We similarly initialize their generators as identity functions.
\end{inlinelist}

\begin{figure}[!t]
\centering
     \subfloat[Data samples]{%
       \includegraphics[width=0.48\textwidth]{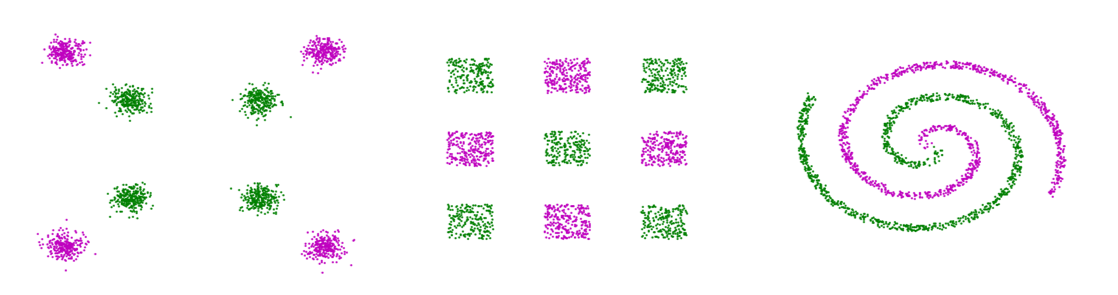}
     }
     \hfill
     \subfloat[Discrete-OT]{%
       \includegraphics[width=0.48\textwidth]{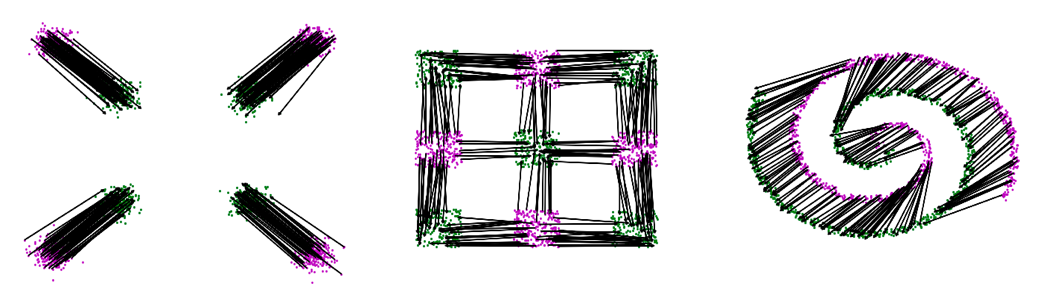}
     }
     \caption{2D data experimental settings.
(a) 1024 data samples of $P_z$ (magenta) and $P_x$ (green) in 4-Gaussians, Checkerboard, 2-Spirals. (b) Optimal maps (black arrows) with Discrete-OT \emph{between 200 samples from $P_z$ and $P_x$}}
     \label{fig:2d_data}
   \end{figure}
   

\newcommand{\figwidth}{0.21}

\begin{figure*}[!ht]

\centering
\begin{tabular}{C{3.1cm}C{3.1cm}C{3.1cm}C{3.1cm}}
  \begin{minipage}[t]{\figwidth\textwidth}
    \centering
    \includegraphics[width=\textwidth]{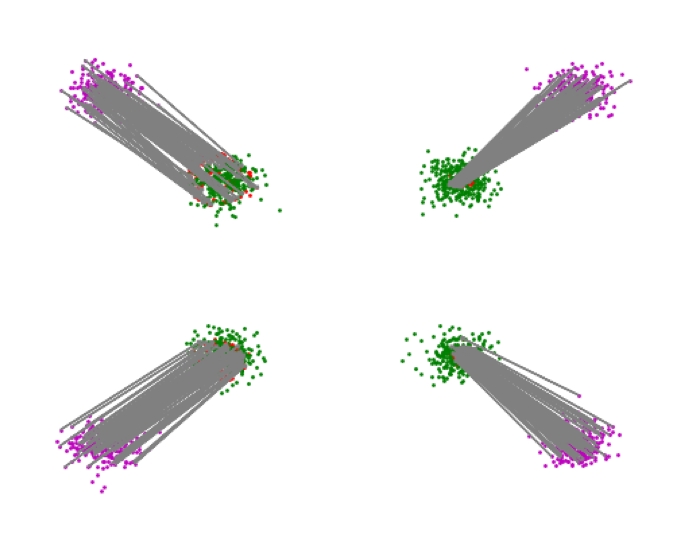} 
  \end{minipage}&
  \begin{minipage}[t]{\figwidth\textwidth}
    \centering
    \includegraphics[width=\textwidth]{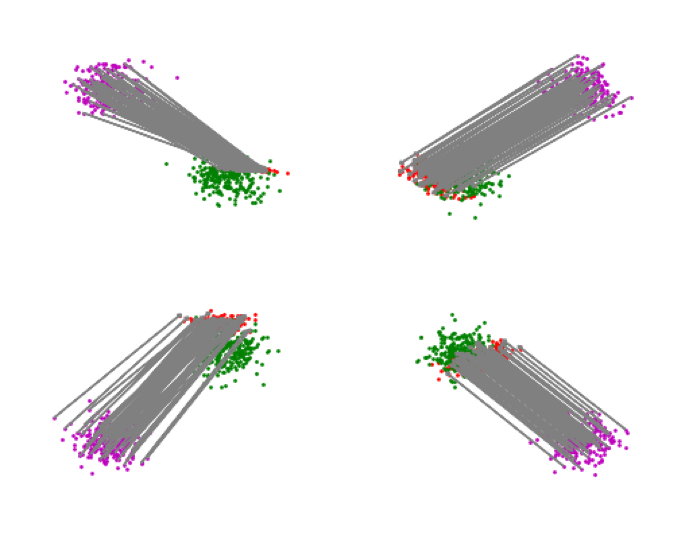} 
  \end{minipage}&
  \begin{minipage}[t]{\figwidth\textwidth}
    \centering
    \includegraphics[width=\textwidth]{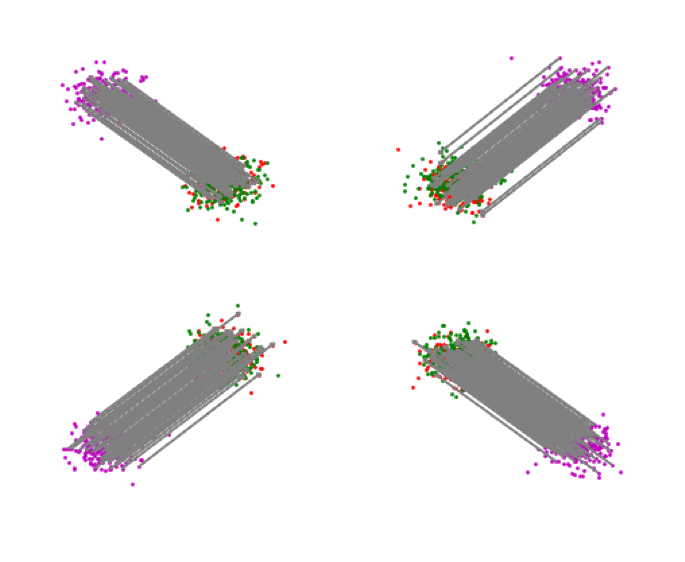}
  \end{minipage}&
  \begin{minipage}[t]{\figwidth\textwidth}
    \centering
    \includegraphics[width=\textwidth]{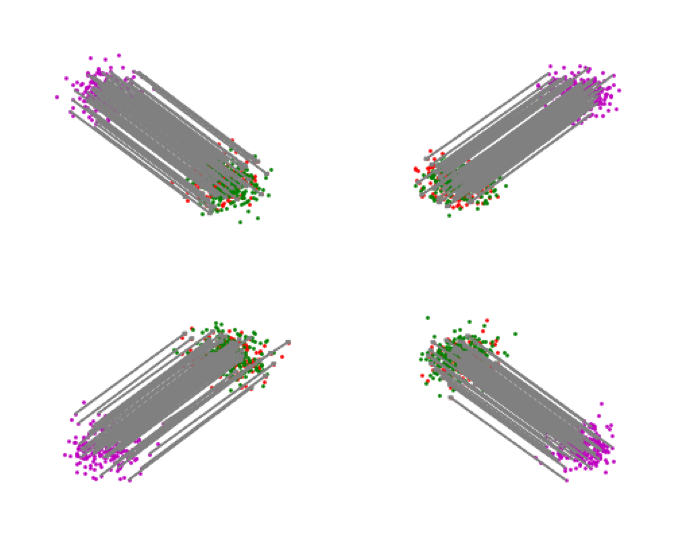}
  \end{minipage}\\
 \begin{minipage}[t]{\figwidth\textwidth}
    \centering
    \includegraphics[width=\textwidth]{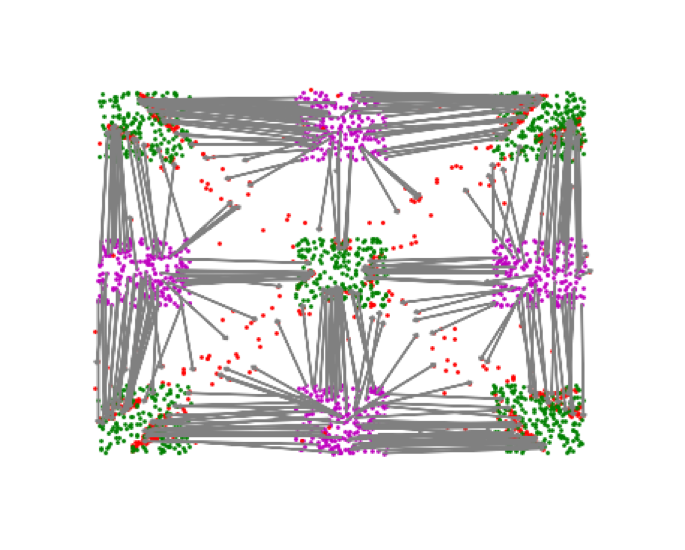} 
  \end{minipage}&
  \begin{minipage}[t]{\figwidth\textwidth}
    \centering
    \includegraphics[width=\textwidth]{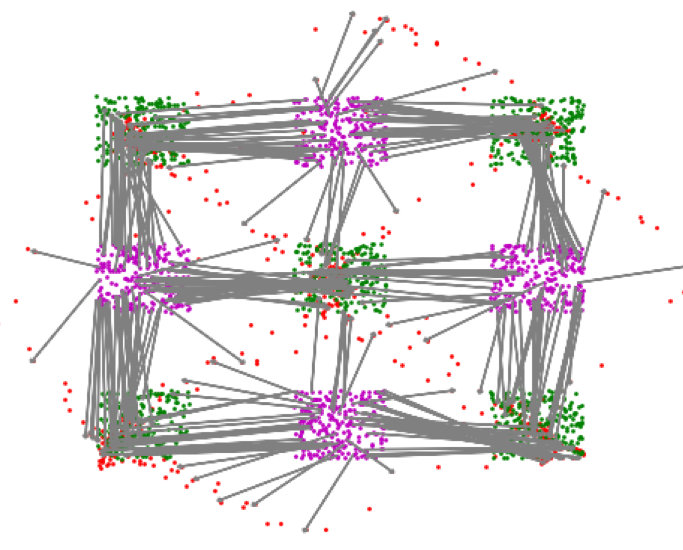} 
  \end{minipage}&
  \begin{minipage}[t]{\figwidth\textwidth}
    \centering
    \includegraphics[width=\textwidth]{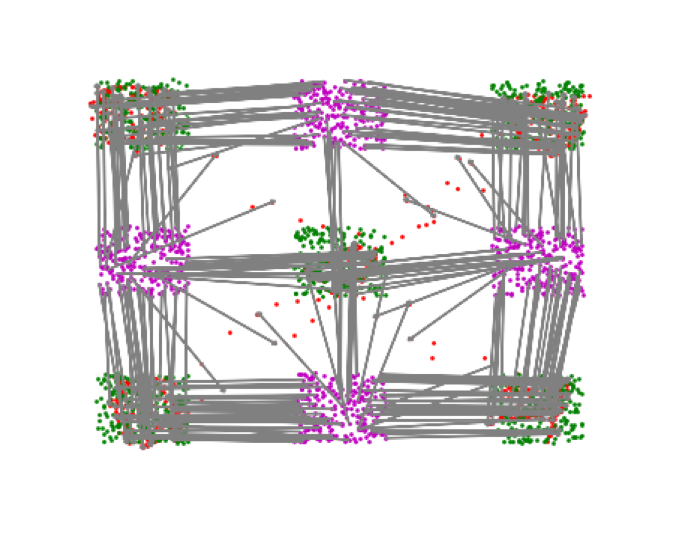} 
  \end{minipage}&
  \begin{minipage}[t]{\figwidth\textwidth}
    \centering
    \includegraphics[width=\textwidth]{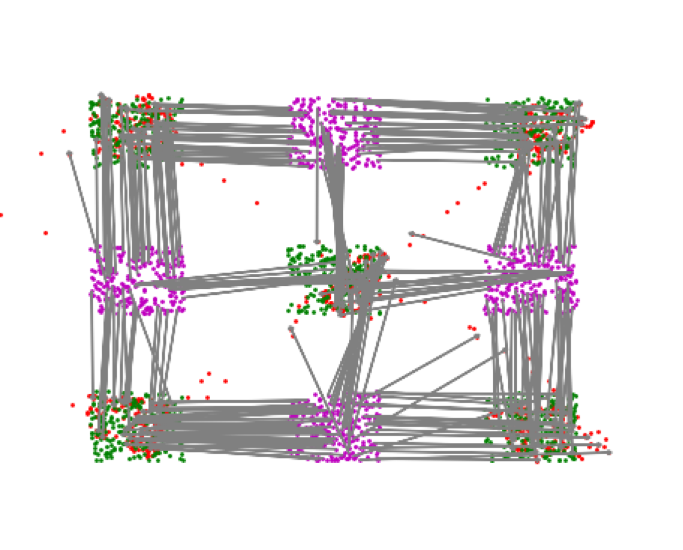}
  \end{minipage}\\
  \begin{minipage}[t]{\figwidth\textwidth}
    \centering
    \includegraphics[width=\textwidth]{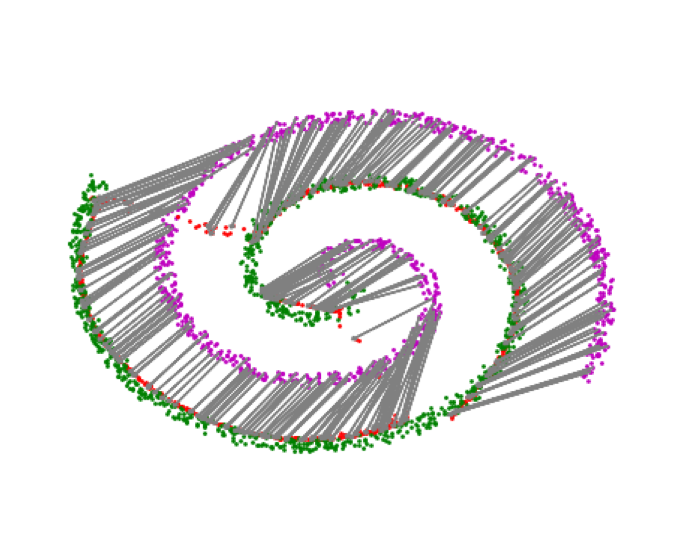} 
  \end{minipage}&
  \begin{minipage}[t]{\figwidth\textwidth}
    \centering
    \includegraphics[width=\textwidth]{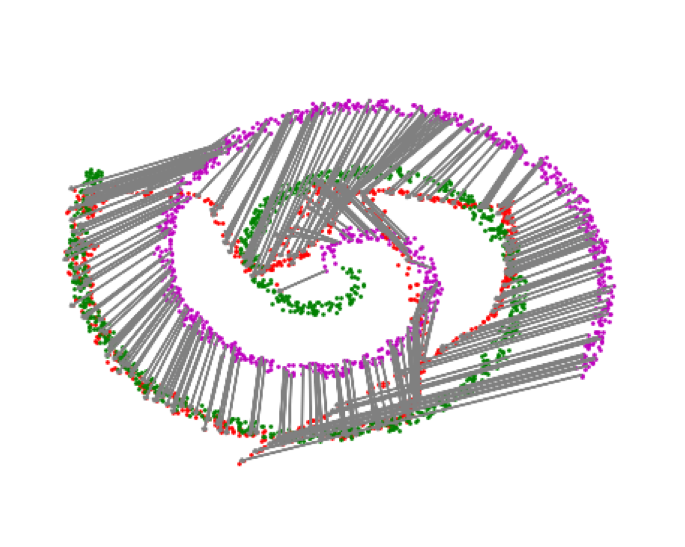} 
  \end{minipage}&
  \begin{minipage}[t]{\figwidth\textwidth}
    \centering
    \includegraphics[width=\textwidth]{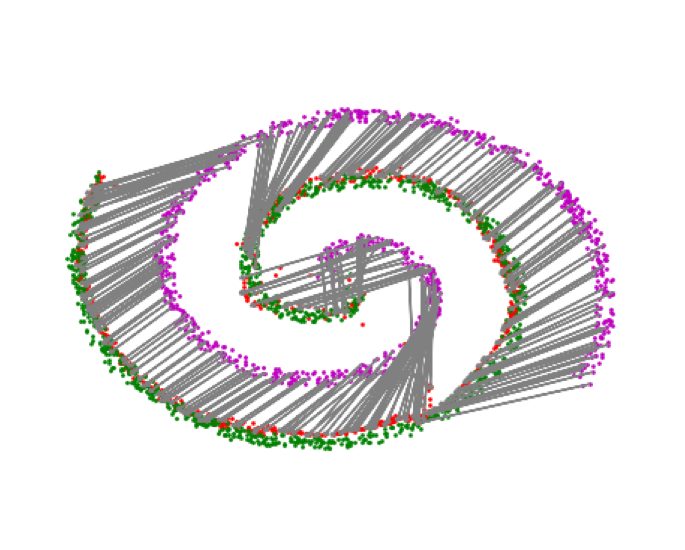}
  \end{minipage}&
  \begin{minipage}[t]{\figwidth\textwidth}
    \centering
     \includegraphics[width=\textwidth]{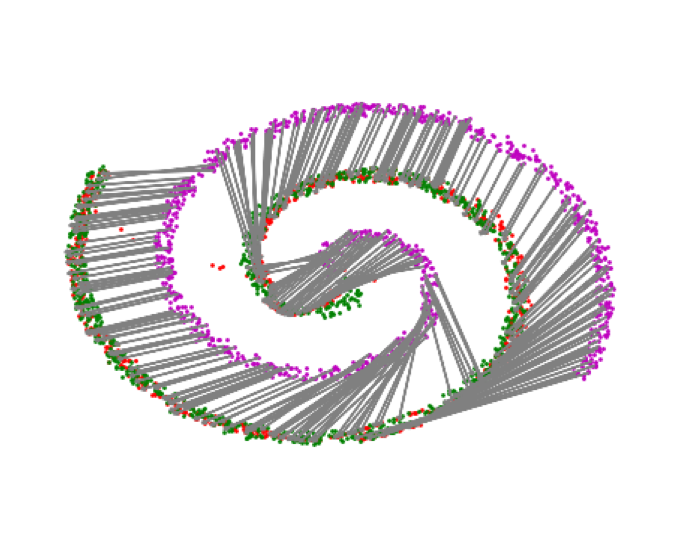} 
  \end{minipage}\\
      \small{(a) Barycentric-OT} & \small{(b) WGAN-GP} & \small{(c) WGAN-LP} & \small{(d) W2GAN (ours)} \\
\end{tabular}
\caption{Maps learned in three synthetic datasets (Figure~\ref{fig:2d_data}) by Barycentric-OT, WGAN-GP, WGAN-LP and our proposed W2GAN.
The generator maps (with gray arrows) samples $z\sim P_z$ (magenta) to $G(z)$ (red) such that it matches $P_x$ (green).
}
\label{fig:gan_2d}
\end{figure*}


\begin{figure}[!t]
\centering
  \subfloat[$x - \nabla \phi(x)$]{
    \includegraphics[width=0.2\textwidth]{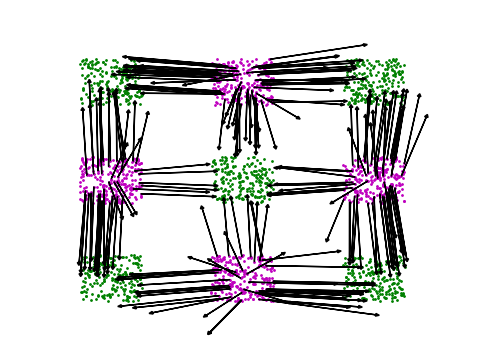} 
    }
    \subfloat[$\phi(x)$]{
    \includegraphics[width=0.2\textwidth]{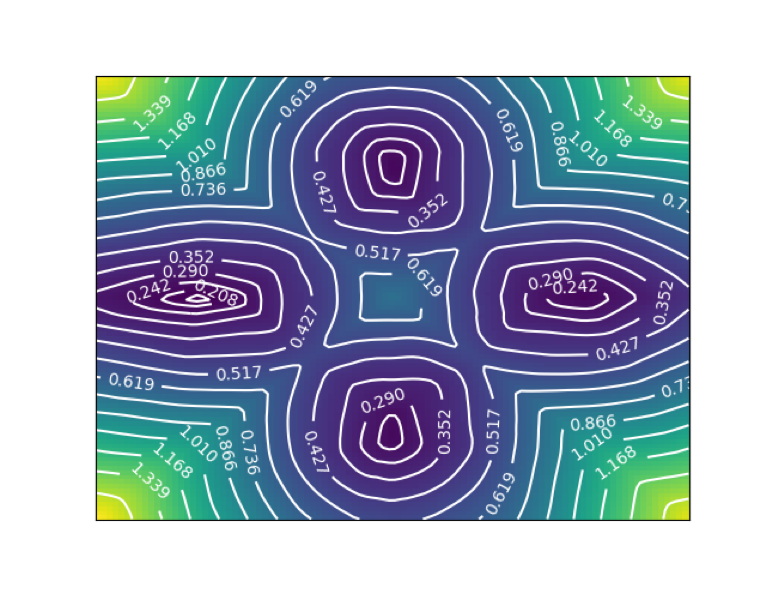} 
    }
    \subfloat[$\psi(x)$]{
    \includegraphics[width=0.2\textwidth]{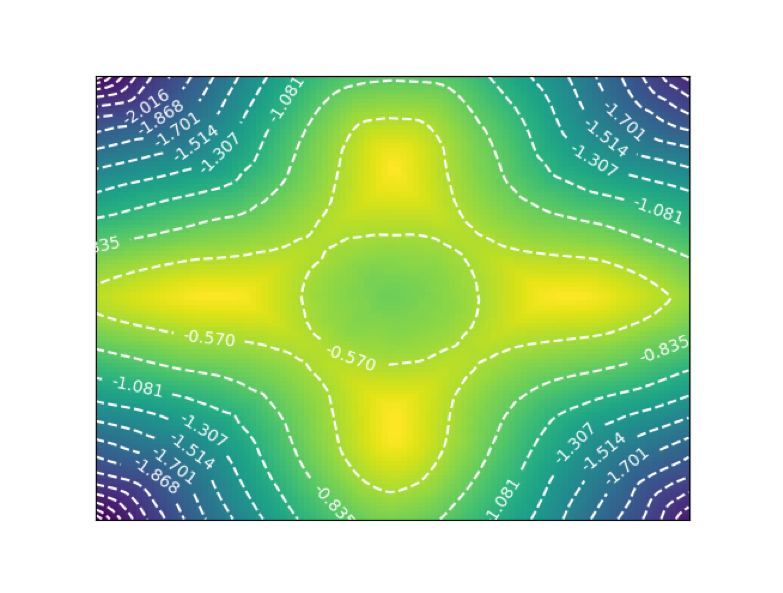} 
    }
\caption{The discriminator approximates the Monge map locally in W2GAN. (a) Gradient direction provided by $\phi$ to the generator. 
(b-c) Heat-map of values of $\phi$ and $\psi$ over $\mathds{R}^2$. }
\label{fig:local_OT}
\end{figure}

\begin{figure*}[t!]

\centering
\begin{tabular}{C{3cm}C{3cm}C{3cm}C{3cm}C{3cm}}
  \begin{minipage}[t]{\figwidth\textwidth}
    \centering
    \includegraphics[width=\textwidth]{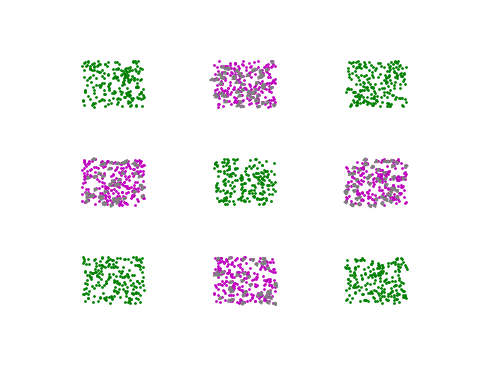} 
  \end{minipage}&
  \begin{minipage}[t]{\figwidth\textwidth}
    \centering
    \includegraphics[width=\textwidth]{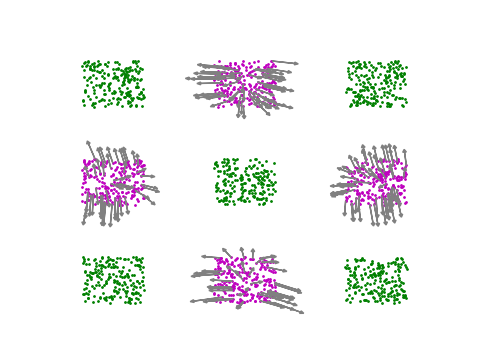} 
  \end{minipage}&
  \begin{minipage}[t]{\figwidth\textwidth}
    \centering
    \includegraphics[width=\textwidth]{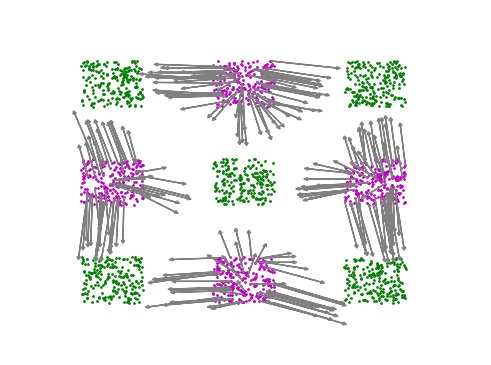}
  \end{minipage}&
  \begin{minipage}[t]{\figwidth\textwidth}
    \centering
    \includegraphics[width=\textwidth]{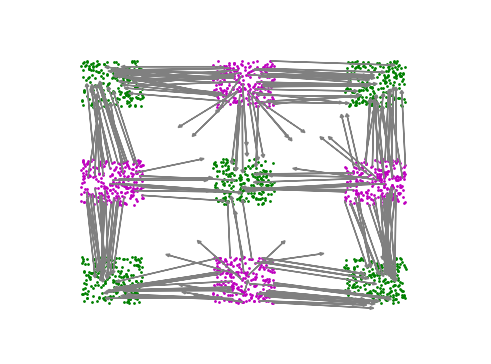}
  \end{minipage}\\
 \begin{minipage}[t]{\figwidth\textwidth}
    \centering
    \includegraphics[width=\textwidth]{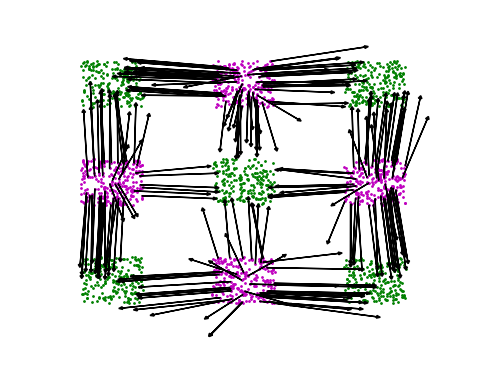} 
  \end{minipage}&
  \begin{minipage}[t]{\figwidth\textwidth}
    \centering
    \includegraphics[width=\textwidth]{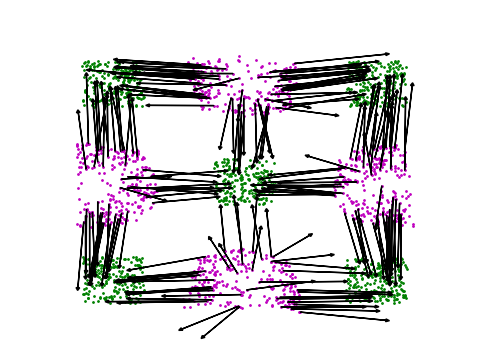} 
  \end{minipage}&
  \begin{minipage}[t]{\figwidth\textwidth}
    \centering
    \includegraphics[width=\textwidth]{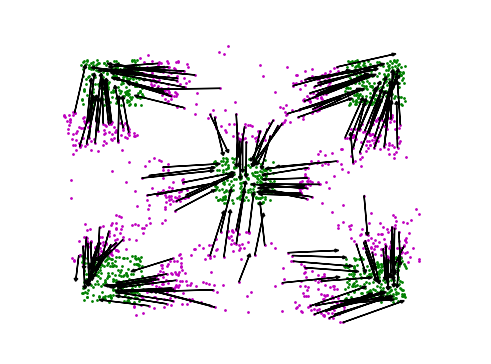}
  \end{minipage}&
  \begin{minipage}[t]{\figwidth\textwidth}
    \centering
    \includegraphics[width=\textwidth]{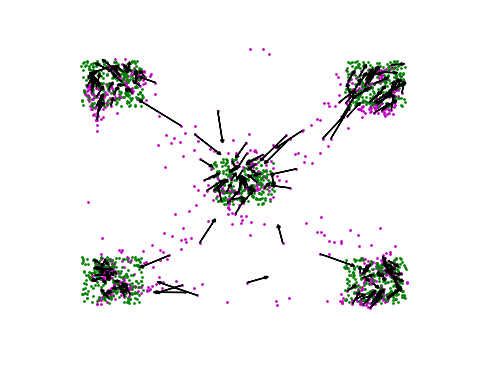}
  \end{minipage}\\
      \small{(a) step 500} & \small{(b) step 3000} & \small{(c) step 5500} & \small{(d) step 8000}
\end{tabular}
\caption{Evolution of the generator (top row) and the gradient it receives from the discriminator $\phi$ (bottom row) through training. The generator starts as an identity function, and is updated in the direction of $G(z)-\nabla\phi(G(z))$ (black arrows) where $z\sim P_z$(pink). The generator is following the $W_2$-geodesic.}


\label{fig:gan_evolution}
\end{figure*}


We show in Figure~\ref{fig:gan_2d} the performance of each of these models on the three 2D datasets.
When compared to the Discrete-OT map in Figure~\ref{fig:2d_data}-(b), we can see that
our proposed W2GAN can successfully recover the Monge map in all three datasets.
In comparison, Barycentric-OT performs mostly very well, but we notice a clear collapse of the mapping in the case of 4-Gaussians. WGAN-GP generally performs poorly, and we observe that training becomes especially unstable as the generated distribution approaches $P_x$, sometimes diverging from a good solution. This happens even with more training of the discriminator and smaller generator learning rates. WGAN-LP, on the other hand, matches the performance of W2GAN in all three datasets.\footnote{WGAN-LP clearly outperforms WGAN-GP, which can be unstable because the gradient penalty in WGAN-GP can prevent the discriminator from converging~\citep{petzka2018on}.} The fact this optimal map seems to match the Monge map of the $W_2$ case is surprising, because in theory there are many $W_1$ geodesics which the generator can follow. We think the bias of the parameters makes WGAN-LP's generator follow a specific geodesic.  


We also verify our analysis that the local direction provided to the generator is given by a perfect discriminator $\phi$, i.e that $\phi$ is related to the Monge map $T$ according to 
Eqn~\eqref{T_GRAD_PHI}. We can see in Figure~\ref{fig:local_OT} that the approximation of the Monge map by $\phi$ is quite reasonable. This is especially interesting, given that $\phi$, as well as $\psi$, are both real functions which are not explicitly trained to recover the Monge map. Finally, Figure~\ref{fig:gan_evolution} shows the generator's training sequence for the checkerboard dataset. We see that the generated distribution evolves along the optimal map.

\subsection{High-dimensional Data}

 Next, we move to a more challenging setting where we learn optimal maps in high-dimensional image data. 
 We first consider the task of mapping samples of a $28\times 28$ multivariate Gaussian\footnote{With mean and covariance matrix estimated with maximum-likelihood on MNIST training data.} to MNIST, which was originally proposed by~\mycitet{Seguy2018LSOT}, and compare with their method. We follow the exact experimental setting of~\mycitet{Seguy2018LSOT}, where we use the same architecture for both W2GAN generator and their Barycentric-OT.~\footnote{More experimental details can be found in the Appendix~\ref{sec:exp_details}.}
 Figure~\ref{fig:mvg_to_mnist} shows the generated samples by our model and the Barycentric-OT model. Qualitatively speaking, W2GAN seems to generate much better MNIST samples, which confirms that it is a competitive method for estimating the Monge maps in high dimensions.

In the second set of experiments, we apply our model to the unsupervised domain adaptation (UDA) task, which is a standard experimentation setting for evaluating large-scale OT maps~\citep{courty2017optimal}. We also follow the same experimental settings of~\mycitet{Seguy2018LSOT} for UDA across USPS and MNIST datasets. 
In UDA, we have labels in one dataset (e.g. USPS) and our goal is to train a classifier which performs well on the other dataset (e.g. MNIST) \emph{without} having access to any training labels in it. An optimal map is assumed to learn a map in the image space that preserves, as much as possible, the digit identity.
 
Similar to~\mycitet{Seguy2018LSOT}, we evaluate the accuracy of different models using a 1-nearest neighbour (1-NN) classifier. As a baseline, we train the classifier using USPS (resp. MNIST) images, and test it on USPS (resp. MNIST) images. 
We compare the classification accuracy of this naive model to the one obtained by both our W2GAN model, Barycentric-OT, and WGAN-LP, where we train the 1-NN classifier on the transferred images using labels from the source. We can see in Table~\ref{tab:da_results} that our model always achieves better results than Barycentric-OT. We note that our baseline is slightly lower than that of~\mycitet{Seguy2018LSOT}, possibly due to image-processing differences, and correspondingly, our implementation of Barycentric-OT also achieves slightly lower accuracy.  Even so, we can achieve higher accuracy with our model compared to the result reported in their paper. The WGAN-LP baseline, on the other hand, is competitive with our W2GAN model.


\begin{figure}[!t]
\centering
\subfloat[MV-Gaussian inputs]{
\includegraphics[width=0.3\textwidth]{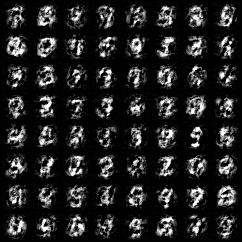}
}
\hfill
\subfloat[Barycentric-OT outputs]{
\includegraphics[width=0.3\textwidth]{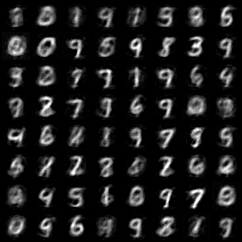}
}
\hfill
\subfloat[W2GAN outputs]{
\includegraphics[width=0.3\textwidth]{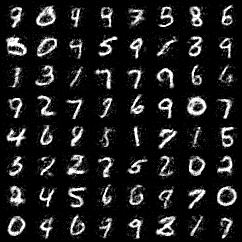}
}
\caption{Mapping MV-Gaussian to MNIST. (a) Samples from MV-Gaussian with mean and covariance estimated from MNIST. (b) Corresponding samples from Barycentric-OT. (c) Corresponding samples from our W2GAN model. }
\label{fig:mvg_to_mnist}
\end{figure}

\begin{table}[!t]
\centering
\begin{tabular}{ C{3cm}C{3cm}C{3cm} } 
 \toprule
  & USPS$\rightarrow$MNIST & MNIST$\rightarrow$USPS\\ 
 \midrule
 Source & 33.31 & 70.05\\ 
 Barycentric-OT$^\dagger$ & 60.50 & 77.92\\ 
 Barycentric-OT$^\ddagger$ & 58.68 & 63.48\\ 
 WGAN-LP &  59.34  & \textbf{81.71}\\ 
 W2GAN & \textbf{67.89} & 80.02\\ 
 \bottomrule \\
\end{tabular}
\caption{Accuracy of 1-NN classifier (in \%) trained using source vs. transported data. $^\dagger$From~\citep{Seguy2018LSOT}. $^\ddagger$Our implementation.}
    \label{tab:da_results}
\end{table}

\section{Conclusion}
We believe this work offers a new perspective on GANs: a way to characterize the dynamics of the generator during training, and as a main application, a way to compute a Monge map. In W2GAN, the generator recovers an optimal map between its initial and target distributions. To establish this, we connected the generator's training sequence, and especially the signal it receives from the discriminator, with the $W_2$ geodesics. Extending this analysis to other types of GAN would be of great interest. In Appendix~\ref{open_questions}, we raise other questions about our model left for future work.


%
%
%
\bibliographystyle{splncs04} 
\bibliography{references}

\begin{thebibliography}{10}
\providecommand{\url}[1]{\texttt{#1}}
\providecommand{\urlprefix}{URL }
\providecommand{\doi}[1]{https://doi.org/#1}

\bibitem{almahairi2018augmented}
Almahairi, A., Rajeshwar, S., Sordoni, A., Bachman, P., Courville, A.:
  Augmented cyclegan: Learning many-to-many mappings from unpaired data. In:
  ICML. pp. 195--204 (2018)

\bibitem{Ambrosio2013}
Ambrosio, L., Gigli, N.: A user’s guide to optimal transport. In: Modelling
  and optimisation of flows on networks, pp. 1--155. Springer (2013)

\bibitem{ambrosio2008gradient}
Ambrosio, L., Gigli, N., Savar{\'e}, G.: Gradient flows: in metric spaces and
  in the space of probability measures. Springer Science \& Business Media
  (2008)

\bibitem{arjovsky2017wasserstein}
Arjovsky, M., Chintala, S., Bottou, L.: Wasserstein gan. arXiv preprint
  arXiv:1701.07875  (2017)

\bibitem{Blondel2018}
Blondel, M., Seguy, V., Rolet, A.: Smooth and sparse optimal transport. arXiv
  preprint arXiv:1710.06276  (2017)

\bibitem{bottou2018geometrical}
Bottou, L., Arjovsky, M., Lopez-Paz, D., Oquab, M.: Geometrical insights for
  implicit generative modeling. In: Braverman Readings in Machine Learning. Key
  Ideas from Inception to Current State, pp. 229--268. Springer (2018)

\bibitem{brenier1991polar}
Brenier, Y.: Polar factorization and monotone rearrangement of vector-valued
  functions. Communications on pure and applied mathematics  \textbf{44}(4),
  375--417 (1991)

\bibitem{Cominetti1994}
Cominetti, R., San~Mart{\'\i}n, J.: Asymptotic analysis of the exponential
  penalty trajectory in linear programming. Mathematical Programming
  \textbf{67}(1-3),  169--187 (1994)

\bibitem{courty2017optimal}
Courty, N., Flamary, R., Tuia, D., Rakotomamonjy, A.: Optimal transport for
  domain adaptation. IEEE transactions on pattern analysis and machine
  intelligence  \textbf{39}(9),  1853--1865 (2017)

\bibitem{cuturi2013sinkhorn}
Cuturi, M.: Sinkhorn distances: Lightspeed computation of optimal transport.
  In: NIPS. pp. 2292--2300 (2013)

\bibitem{1310.4375}
Cuturi, M., Doucet, A.: Fast computation of wasserstein barycenters. In: ICML.
  pp. 685--693 (2014)

\bibitem{Genevay2016}
Genevay, A., Cuturi, M., Peyr{\'e}, G., Bach, F.: Stochastic optimization for
  large-scale optimal transport. In: NIPS. pp. 3440--3448 (2016)

\bibitem{genevay2017learning}
Genevay, A., Peyr{\'e}, G., Cuturi, M.: Learning generative models with
  sinkhorn divergences. arXiv preprint arXiv:1706.00292  (2017)

\bibitem{goodfellow2014generative}
Goodfellow, I., Pouget-Abadie, J., Mirza, M., Xu, B., Warde-Farley, D., Ozair,
  S., Courville, A., Bengio, Y.: Generative adversarial nets. In: NIPS. pp.
  2672--2680 (2014)

\bibitem{gramfort2015fast}
Gramfort, A., Peyr{\'e}, G., Cuturi, M.: Fast optimal transport averaging of
  neuroimaging data. In: International Conference on Information Processing in
  Medical Imaging. pp. 261--272. Springer (2015)

\bibitem{gulrajani2017improved}
Gulrajani, I., Ahmed, F., Arjovsky, M., Dumoulin, V., Courville, A.C.: Improved
  training of wasserstein gans. In: NIPS. pp. 5767--5777 (2017)

\bibitem{johnson2018composite}
Johnson, R., Zhang, T.: Composite functional gradient learning of generative
  adversarial models. In: ICML. pp. 2376--2384 (2018)

\bibitem{kodali2017convergence}
Kodali, N., Abernethy, J., Hays, J., Kira, Z.: On convergence and stability of
  gans. arXiv preprint arXiv:1705.07215  (2017)

\bibitem{Lei2017AGV}
Lei, N., Su, K., Cui, L., Yau, S.T., Gu, D.X.: A geometric view of optimal
  transportation and generative model. arXiv preprint arXiv:1710.05488  (2017)

\bibitem{lin2018pacgan}
Lin, Z., Khetan, A., Fanti, G., Oh, S.: Pacgan: The power of two samples in
  generative adversarial networks. In: NIPS. pp. 1505--1514 (2018)

\bibitem{lu2018guiding}
Lu, G., Zhou, Z., Song, Y., Ren, K., Yu, Y.: Guiding the one-to-one mapping in
  cyclegan via optimal transport. AAAI  (2019)

\bibitem{lui2017implicit}
Lui, K.Y.C., Cao, Y., Gazeau, M., Zhang, K.S.: Implicit manifold learning on
  generative adversarial networks. arXiv preprint arXiv:1710.11260  (2017)

\bibitem{mescheder2018training}
Mescheder, L., Geiger, A., Nowozin, S.: Which training methods for gans do
  actually converge? In: ICML. pp. 3478--3487 (2018)

\bibitem{Miyato2018spectral}
Miyato, T., Kataoka, T., Koyama, M., Yoshida, Y.: Spectral normalization for
  generative adversarial networks. ICLR  (2018)

\bibitem{nagarajan2017gradient}
Nagarajan, V., Kolter, J.Z.: Gradient descent gan optimization is locally
  stable. In: NIPS. pp. 5585--5595 (2017)

\bibitem{NIPS2016_6066}
Nowozin, S., Cseke, B., Tomioka, R.: f-gan: Training generative neural samplers
  using variational divergence minimization. In: NIPS. pp. 271--279 (2016)

\bibitem{petzka2018on}
Petzka, H., Fischer, A., Lukovnikov, D.: On the regularization of wasserstein
  {GAN}s. In: ICLR (2018)

\bibitem{Peyre2018}
Peyr{\'e}, G., Cuturi, M., et~al.: Computational optimal transport. Foundations
  and Trends in Machine Learning  \textbf{11}(5-6),  355--607 (2019)

\bibitem{Salimans2017otgan}
Salimans, T., Zhang, H., Radford, A., Metaxas, D.: Improving gans using optimal
  transport. arXiv preprint arXiv:1803.05573  (2018)

\bibitem{Sanjabi2018swgan}
Sanjabi, M., Ba, J., Razaviyayn, M., D.~Lee, J.: On the convergence and
  robustness of training {GAN}s with regularized optimal transport. arXiv
  preprint arXiv:1802.08249  (2018)

\bibitem{santambrogio2017euclidean}
Santambrogio, F.: $\{$Euclidean, metric, and Wasserstein$\}$ gradient flows: an
  overview. Bulletin of Mathematical Sciences  \textbf{7}(1),  87--154 (2017)

\bibitem{Seguy2018LSOT}
Seguy, V., Bhushan~Damodaran, B., Flamary, R., Courty, N., Rolet, A., Blondel,
  M.: Large-scale optimal transport and mapping estimation. arXiv preprint
  arXiv:1711.02283  (2018)

\bibitem{solomon2015convolutional}
Solomon, J., De~Goes, F., Peyr{\'e}, G., Cuturi, M., Butscher, A., Nguyen, A.,
  Du, T., Guibas, L.: Convolutional wasserstein distances: Efficient optimal
  transportation on geometric domains. ACM Transactions on Graphics (TOG)
  \textbf{34}(4), ~66 (2015)

\bibitem{Villani}
Villani, C.: Optimal transport, old and new. A Series of comprehensive Studies
  in Mathematics  (2008)

\bibitem{xie2018fast}
Xie, Y., Wang, X., Wang, R., Zha, H.: A fast proximal point method for
  wasserstein distance. arXiv preprint arXiv:1802.04307  (2018)

\bibitem{yamaguchi2018distributional}
Yamaguchi, S., Koyama, M.: Distributional concavity regularization for gans.
  In: ICLR (2019)

\bibitem{zhu2017unpaired}
Zhu, J.Y., Park, T., Isola, P., Efros, A.A.: Unpaired image-to-image
  translation using cycle-consistent adversarial networks. arXiv preprint
  arXiv:1703.10593  (2017)

\end{thebibliography}

%
%


\appendix

\section{Appendix for Adversarial Computation of Optimal Transport Maps}

\subsection{Remaining questions and future work}
\label{open_questions}
In the following we provide some future questions which are raised by our proposed method and theoretical analysis: 
\begin{itemize}
    \item A crucial relationship used in our model is~\eqref{T_GRAD_PHI}, relating an optimal map $T$ solving the Monge problem~\eqref{MONGE_PROBLEM} with the Kantorovitch potential $\phi$ solving the dual~\eqref{KANTOROVITCH_DUAL}. We recall that this link is $T(x)=x-\nabla \phi(x)$ for any $x$ in the support of the initial distribution. In practice though, we get this potential by solving a regularized version of the Kantorovitch problem, for instance~\eqref{LOSS_DISCRIMINATOR_GAN_INEQONLY} or~\eqref{DISCRIMINATOR_LOSS_EQ_INEQ}. Assuming for simplicity we deal with~\eqref{LOSS_DISCRIMINATOR_GAN_INEQONLY}, we face a dual problem which basically removes the hard constraint defining~\eqref{KANTOROVITCH_DUAL} by including a penalty term $\mathcal{L}_{\text{ineq}}$ in the main objective, the emphasis of which we control with the hyperparameter $\lambda_{\text{ineq}}$. Denoting $\phi_{\lambda_{\text{ineq}}}$ the corresponding solution -which may be proven to be unique up to translation- it seems to constitute a hard challenge to prove that $\phi_{\lambda_{\text{ineq}}} \rightarrow \phi$ as $\lambda_{\text{ineq}} \rightarrow 0$. The case of the entropy penalty, that is when we define the regularization $\mathcal{L}_{\text{ineq}}(\phi,\psi):=- \mathds{E}_{(x,y) \sim \mu \times \nu} [\exp(\frac{\phi(x)+\psi(y)-c(x,y)}{\lambda_{ineq}})]$ has been widely studied. In a discrete setting, the convergence might be obtained as a result of~\mycitet{Cominetti1994}, but their analysis does not straightforwardly scale to the case of absolutely continuous probability measures. What we would find more useful is to prove the convergence $\nabla \phi_{\lambda_{\text{ineq}}} \rightarrow \nabla \phi$, as it would mean that our approximation of the optimal transport map $T$ using relationship~\eqref{T_GRAD_PHI} is valid. Notice that the entropy regularized dual also admits a primal formulation, which solution $\pi_{\lambda_{\text{ineq}}}$ may be proven to converge in some sense to the unique transport map $\pi$ solving~\eqref{KANTOROVITCH_PRIMAL}. No similar results have been found about the dual variables for the moment. Also, the case of other penalties as the $L_2$ one remain open problems. 
    \item In a similar fashion, it would enrich our analysis to have bounds on the deviation from the ideal trajectory, the $W_2$-geodesic, when we assume some errors of approximation for the generator and the discriminator. In the same way, how far is the final generated distribution from the Monge map in this more realistic case? 
    \item Another crucial theoretical need for strengthening our analysis would be to deepen the parametric analysis~\ref{parametric analysis} and try to understand the trajectory of the generated distribution in the parameter space more thoroughly. Ideally, one would like to prove that this trajectory would be some projection of the $W_2$-geodesic in the space of parameters. 
    \item In practice, what are the new possibilities enabled by the analysis of the generator's dynamic and the W2GAN model? First, it would be of interest to characterize the dynamic of GANs relying of $f$-divergences in a similar manner. In the Wasserstein-1 case, it is possible that knowing the generator is learning an optimal correspondence could be useful in a domain transfer situation, especially in high dimension where those models perform well. In particular, listing the different $W_1$-geodesics and characterizing the ones that are mainly followed by WGAN's generators could be a great insight.
    \item On the other hand, as the W2GAN's discriminator approximates the second Wasserstein distance, it enables tackling the challenge of computing Wasserstein barycenters of distributions~\citep{1310.4375}. In fact, it seems likely that an adversarial and parametric method would perform well in the task of generating all Wasserstein interpolations between distributions in high-dimension, a difficult problem addressed by~\mycitet{xie2018fast} for instance.
\end{itemize}

\subsection{Analysis of the parametric update rule}

\label{parametric analysis}
This section is dedicated to the analysis of the parametric update. \mycitet{lui2017implicit} already made the conjecture that a generator trained with the second Wasserstein distance would have its parameters updated towards the direction of an optimal transport. Let us prove their statement in the case of our model. In the following, $J_{f}(u)$ denotes the Jacobian of a function $f$ at point $u$. Let $G(\theta,z):=G_\theta(z)$ be our parametrized generator, which we assume to admit uniformly bounded derivatives w.r.t both the parameters $\theta$ and $z$. Recall the notation $\mu_\theta:=\mathds{P}_{G_\theta(z)}=G_\theta\#\mathds{P}_z$ where $\mathds{P}_z$ is the measure known beforehand on the input variable $z$. We consider the context of alternating gradient descent with learning rate $\alpha>0$ for the generator. We naturally associate the fictive time variable t, so that we consider the discrete update equation:
\begin{equation}
   \theta_{t+1} = \theta_t -\alpha \frac{\partial{W_2^2(\mu_{\theta_t},P_x)}}{\partial \theta}
  \label{UPDATE_THETA}
\end{equation}

\begin{proposition}\label{prop:dynamic}
\normalfont At each generator update, we assume the discriminator $\phi,\psi$ to achieve $\sup_{\phi,\psi} V^*(\phi,\psi,\mu_\theta,P_x)=W_2^2(\mu_\theta,P_x)$ where $V^*(\phi,\psi,\mu_\theta,P_x):= \int_{\mathds{R}^d} \phi(x)dP(x) + \int_{\mathds{R}^d} \psi(x)dP_x$ is the value function of the dual Kantorovitch problem~\eqref{KANTOROVITCH_DUAL}. Then~\eqref{UPDATE_THETA} admits a tractable form as:
\[\frac{\partial{ W_2^2}}{\partial \theta}= \mathds{E}_{z \sim P_z}(J_{\phi(G_\theta(z))}^T (\theta))= \mathds{E}_{z\sim P_z}(\nabla_\theta \phi(G_\theta(z)))\]
Moreover, the gradient ascent dynamics of the generator are linked to the optimal transport map $T$ solving~\eqref{MONGE_PROBLEM} between $\mu_\theta$ and $P_x$ as:
\begin{align}
\forall z,& \; G(\theta_{t+1},z)= G(\theta_t,z) \nonumber \\+\; &\alpha J_{G(.,z)}(\theta_t) \times \mathds{E}_{z\sim P_z}\left( J_{G(.,z)}^T(\theta_t)\times \left[T(G(\theta_t,z))-G(\theta_t,z) \right]\right) \nonumber \\+\; &o(\alpha)
\label{IDEAL_UPDATE_G}
\end{align}
\end{proposition}
In the above, we use the notation $G(\theta,z):=G_\theta(z)$ to clarify where the derivatives are taken. Equation~\eqref{IDEAL_UPDATE_G} is close to the one conjectured in~\mycitet{lui2017implicit}. Following their approach, for the sake of clarity, we consider the case where the input variable $z$ is a constant. Then we read~\eqref{IDEAL_UPDATE_G} as 
\begin{align}
G(\theta_{t+1}&,z)= G(\theta_t,z) \nonumber \\+&\alpha J_{G(.,z)}(\theta_t)   J_{G(.,z)}^T(\theta_t)\times \left[T(G(\theta_t,z))-G(\theta_t,z) \right] 
\end{align}

In equation \eqref{IDEAL_UPDATE_G}, the term $T(G_{\theta_t}(z'))-G_{\theta_t}(z')$ confirms that locally the generator is updated toward the target distribution along the optimal map. But now this term is weighted by a Jacobian and put under an expectation over the input variable $z$. As observed by \mycitet{lui2017implicit} it is difficult to understand the effect of these additional terms.

\subsection{Proofs of main text's propositions}
\label{sec:proofs_main_text}
We here recall chronologically the different propositions of the main discussion and give their proofs. A more detailed introduction to essential notions of Optimal Transport theory that we need is given in the Appendix \ref{OT_THEORY_BACKGROUND}. 

For two fixed absolutely continuous measures $\mu$ and $\nu$ on $\mathds{R}^m$ and two $L_2$ functions $f,g:\mathds{R}^m \rightarrow \mathds{R}$ denote by $V^*(f,g)$ the quantity $\int_{x \in \mathds{R}^m} f(x)d\mu(x)+\int_{y \in \mathds{R}^m} g(y)d\nu(y)$. 
\begin{proposition}
\normalfont
Given the generated distribution $\mu_{\theta_t}$ at time $t\in \mathds{N}$, the updated distribution $\mu_{\theta_{t+1}}$ lies on the $2$-Wasserstein geodesic joining the current generated distribution $\mu_{\theta_t}$ and the target distribution $P_x$, i.e $\mu_{\theta_{t+1}}\in \mathcal{G}(\mu_{\theta_t},P_x)$. Moreover, we can write $G_{\theta_{t+1}}=(1-\alpha)G_{\theta_{t}}+ \alpha T_t \circ G_{\theta_{t}}$ where $T_t$ is the optimal transport map from $\mu_{\theta_t}$ to $P_x$.
\label{local_update}
\end{proposition}
\begin{proof}
We first prove the second part of the proposition. Because the update is ideal \[G_{\theta_{t+1}}=G_{t+1}=G_{\theta_{t}}-\alpha\nabla \phi_t(G_{\theta_{t}})= (\id-\alpha\nabla \phi_t)\circ G_{\theta_{t}}\]
The discriminator $\phi_t$ being perfect it is related to $T_t$ via $T_t= \id-\nabla \phi_t$ from equation \ref{T_GRAD_PHI}. Hence we obtain $G_{\theta_{t+1}}=(1-\alpha)G_{\theta_{t}}+ \alpha T_t \circ G_{\theta_{t}}$. In terms of probability measures, this writes 
\[\mu_{\theta_{t+1}}=[(1-\alpha)\id+ \alpha T_t]\# \mu_{\theta_t}\]
By the characterization of $W_2$ geodesics for absolutely continuous measures we thus have $\mu_{\theta_{t+1}}\in \mathcal{G}(\mu_{\theta_t},P_x)$.
\end{proof}

\begin{proposition}
\normalfont
\label{global_evolution_generated_distrib}
The training dynamic is a subset of the $W_2$-geodesic between the initial distribution $\mu_{\theta_0}$ and the target distribution $P_x$:
\[\{\mu_{\theta_0},...,\mu_{\theta_t},...\} \subset \mathcal{G}(\mu_{\theta_0},P_x)\] 
\end{proposition}
\begin{proof}
We proceed by induction on $t\in \mathds{N}$. Of course $\mu_{\theta_0}\in \mathcal{G}(\mu_{\theta_0},P_x)$. Assume that for some $t\geqslant 0$, $\{\mu_{\theta_0},...,\mu_{\theta_t}\} \subset \mathcal{G}(\mu_{\theta_0},P_x)$ and let us show that $\mu_{\theta_{t+1}}\in \mathcal{G}(\mu_{\theta_0},P_x)$. From the previous proposition, $\mu_{\theta_{t+1}}\in \mathcal{G}(\mu_{\theta_t},P_x)$. Now we use again the uniqueness of displacement interpolation between geodesics~\citep{Villani} and obtain that $\mu_{\theta_t}\in \mathcal{G}(\mu_{\theta_0},P_x) \Rightarrow \mathcal{G}(\mu_{\theta_t},P_x) \subset \mathcal{G}(\mu_{\theta_0},P_x)$. This concludes the induction step. 
\end{proof}

\begin{proposition}
\normalfont
\label{decreasing_distance}
The distance $W_2(\mu_{\theta_t},P_x)$ decreases exponentially fast. That is,
\[\forall t \in \mathds{N}, \ W_2(\mu_{\theta_t},P_x)\leqslant (1-\alpha)^{t}W_2(\mu_{\theta_0},P_x)\]
\end{proposition}
\begin{proof}
We proceed by induction. Assume $W_2(\mu_{\theta_t},P_x)\leqslant (1-\alpha)^{t}W_2(\mu_{\theta_0},P_x)$ for some $t\geqslant 0$ and let us show that $W_2(\mu_{\theta_{t+1}},P_x)\leqslant (1-\alpha)^{t+1}W_2(\mu_{\theta_0},P_x)$. It is evidently enough to show that $W_2(\mu_{\theta_{t+1}},P_x)\leqslant (1-\alpha)W_2(\mu_{\theta_{t}},P_x)$. 
Let $T_t$ denotes the optimal map joining $\mu_{\theta_t}$ and $P_x$. Proposition \ref{local_update} says that $\mu_{\theta_t}=[(1-\alpha)\id + \alpha T_t]\# \mu_{\theta_t}$. For $T_t$ to be the gradient of a strictly convex function $T_t^{-1}$ makes sense and we can consider the map $\tilde{T}:=[(1-\alpha)\id + \alpha T_t] \circ T_t^{-1}$. By construction $\tilde{T}\# P_x= \mu_{\theta_{t+1}}$. Hence it is an admissible map for the Monge problem between $P_x$ and $\mu_{\theta_{t+1}}$. Let us compute its cost:
\[\int_{x\in \mathds{R}^m} c(x,\tilde{T}(x))dP_x(x)=\frac{1}{2}\int_{x\in \mathds{R}^m} \|x-\tilde{T}(x)\|_2^2dP_x(x)=\frac{(1-\alpha)^2}{2}\int_{x\in \mathds{R}^m} \|(T_t^{-1}(x)-x)\|_2^2dP_x(x)\]
The right-hand side of the above is $(1-\alpha)^2W_2^2(P_x,\mu_{\theta_t})$. We found an admissible map $\tilde{T}$ with such a cost, so by taking the square root of the above and the infimum over admissible map we obtain the desired $W_2(\mu_{\theta_{t+1}},P_x)\leqslant (1-\alpha)W_2(\mu_{\theta_{t}},P_x)$.
\end{proof}

\begin{lemma}
\normalfont
\label{composition_monge_ismonge}
We assume discriminators to be perfect and updates of the generator to be ideal. Then for any integer $t\geqslant 0$,
\begin{enumerate}
\item $G_{\theta_{t+1}}=H_{t,t+1}\circ G_{\theta_{t}}$ where $H_{t,t+1}$ solves the Monge problem between ${\mu_{\theta_t}}$ and $\mu_{\theta_{t+1}}$.
\item Denoting $T_{t,t+k}$ the unique Monge map between $\mu_{\theta_t}$ and $\mu_{\theta_{t+k}}$, we have $T_{t,t+k}=H_{t+k-1,t+k}\circ ... \circ H_{t,t+1}$
\end{enumerate}
\end{lemma}
\begin{proof}
From Proposition \ref{local_update} $H_{t,t+1}$ exists and is written $H_{t,t+1}=(1-\alpha)\id + \alpha T_t$ where $T_t$ is the Monge map between $\mu_{\theta_t}$ and $P_x$. From the fact that $T_t=\id- \nabla \phi_t$, $H_{t,t+1}=\id -\alpha \nabla \phi_t$. Hence $H_{t,t+1}$ remains the gradient of a strictly convex function on $\mathds{R}^m$ and thus is a Monge map. By definition $H_{t,t+1}\#\mu_{\theta_t}=\mu_{\theta_{t+1}}$ and thus the first point is proved. 

To prove the second point, we focus on the case of the map $T_{t,t+2}$, the argument extending by induction for arbitrary $(T_{t,t+k})_{k\in \mathds{N}^*}$. From Proposition \ref{global_evolution_generated_distrib}, $\mu_{\theta_{t+1}}\in \mathcal{G}(\mu_{\theta_{t}},\mu_{\theta_{t+2}})$. From the paragraph about geodesics \ref{sec:geodesics_W_2}, $\mathcal{G}(\mu_{\theta_{t}},\mu_{\theta_{t+2}})=\{(1-\beta)\id+\beta T_{t,t+2})\#\mu_{\theta_{t}}\}_{0\leqslant \beta \leqslant 1}$. Thus there is a $\beta$ such that $((1-\beta)\id+\beta T_{t,t+2})\#\mu_{\theta_{t}}=\mu_{\theta_{t+1}}$. Moreover, as in the previous point, $((1-\beta)\id+\beta T_{t,t+2})$ is an optimal map between $\mu_{\theta_{t}}$ and $\mu_{\theta_{t+1}}$. By uniqueness of Monge maps, $(1-\beta)\id+\beta T_{t,t+2}=H_{t,t+1}$. From the structure of $W_2$-geodesics this also implies that $W_2(\mu_{\theta_t},\mu_{\theta_{t+1}})=\beta W_2(\mu_{\theta_{t}},\mu_{\theta_{t+2}})$, and thus that $W_2(\mu_{\theta_{t+1}},\mu_{\theta_{t+2}})=(1-\beta) W_2(\mu_{\theta_{t}},\mu_{\theta_{t+2}})$.

As gradients of strictly convex functions, Monge maps are injective (in fact invertible as we consider absolutely continuous measures which implies that Monge maps are also surjective). It thus makes sense to consider the function $F:= T_{t,t+2}\circ H_{t,t+1}^{-1}= T_{t+2}\circ [(1-\beta)\id+\beta T_{t,t+2}]^{-1}$. By construction this is an admissible map for the Monge problem between $\mu_{\theta_{t+1}}$ and $\mu_{\theta_{t+2}}$, i.e $F\# \mu_{\theta_{t+1}}=\mu_{\theta_{t+2}}$. From \cite{Villani}, Monge maps are changes of variables thus we have:
\[\int_{y \in \mathds{R}^m} \frac{1}{2}\|F(y)-y\|_2^2 d\mu_{\theta_{t+1}}(y)=\int_{x \in \mathds{R}^m} \frac{1}{2}\|F((1-\beta)x+\beta T_{t,t+2}(x))-(1-\beta)x+\beta T_{t,t+2}(x)\|_2^2 d\mu_{\theta_{t}}(x)\]
\[=\int_{x \in \mathds{R}^m} \frac{1}{2}\|T_{t,t+2}(x)-(1-\beta)x-\beta T_{t,t+2}(x)\|_2^2 d\mu_{\theta_{t}}(x)=(1-\beta)^2\int_{x \in \mathds{R}^m} \frac{1}{2}\|T_{t,t+2}(x)-x\|_2^2 d\mu_{\theta_{t}}(x)\]
The right hand side of the above is exactly $(1-\beta)^2W_2^2(\mu_{\theta_t},\mu_{\theta_{t+2}})$. Hence $F$ is the unique Monge map between $\mu_{\theta_{t+1}}$ and $\mu_{\theta_{t+2}}$, i.e $H_{t+1,t+2}=F$. We have proved that $H_{t+1,t+2}=T_{t,t+2}\circ H_{t,t+1}^{-1}$ which is the desired equality.

\end{proof}

\begin{proposition}
\normalfont
Denote $T$ the optimal map between $\mu_{\theta_0}$ and $P_x$. There is a decreasing function $f:\mathds{N}\rightarrow [0,1]$ with $f(0)=1$ such that
\label{generator_recovers_T}
\[\forall t \in \mathds{N}, \ \mu_{\theta_t}=(f(t)\id+ (1-f(t))T)\# \mu_{\theta_0} \]
\[\forall t \in \mathds{N}, \ G_{\theta_t}=(f(t)\id+ (1-f(t))T)\circ G_{\theta_0} \]
Moreover, $\mu_{\theta_t}$ (resp. $G_{\theta_t}$) converges exponentially fast toward $T \# \mu_{\theta_0}$ (resp. $T\circ G_{\theta_0}$) in the sense that $f(t)\leqslant (1-\alpha)^t$.
\end{proposition}
\begin{proof}
We first prove the part of the proposition that concerns the generated probability measures $\mu_{\theta_t}$ for $t\in \mathds{N}$.
The fact that there exists a function $f:\mathds{N}\rightarrow [0,1]$ with $f(0)=1$ such that 
$\mu_{\theta_t}=(f(t)\id+ (1-f(t))T)\# \mu_{\theta_0}$ for all $t\in \mathds{N}$ is simply using the fact that $\mu_{\theta_t}\in \mathcal{G}(\mu_{\theta_0},P_x)$ and the characterization of $W_2$-geodesics. For a given $t\in \mathds{N}$, again by definition of a $W_2$-geodesic, we have $W_2(\mu_{\theta_t},P_x)=f(t)W_2(\mu_{\theta_0},P_x)$. Hence by proposition \ref{decreasing_distance} $f$ decreases and we have the inequality $f(t)\leqslant (1-\alpha)^t$.

We now address the part of the proposition that concerns the generator function $G_{\theta_t}$. From the lemma \ref{composition_monge_ismonge}, we have $G_{\theta_t}=T_t \circ G_{\theta_0}$ where $T_t$ is the unique Monge map between $\mu_{\theta_0}$ and $\mu_{\theta_t}$. By the above $T_t=(f(t)\id +(1-f(t))T)$ and thus we are done.
\end{proof}
The Figure \ref{Evolution_generator} is a pictorial view of this result that shows how a generated distribution evolves under the $W_2$ loss.
\begin{figure}[t!]
\begin{center}
\includegraphics[width=7cm]{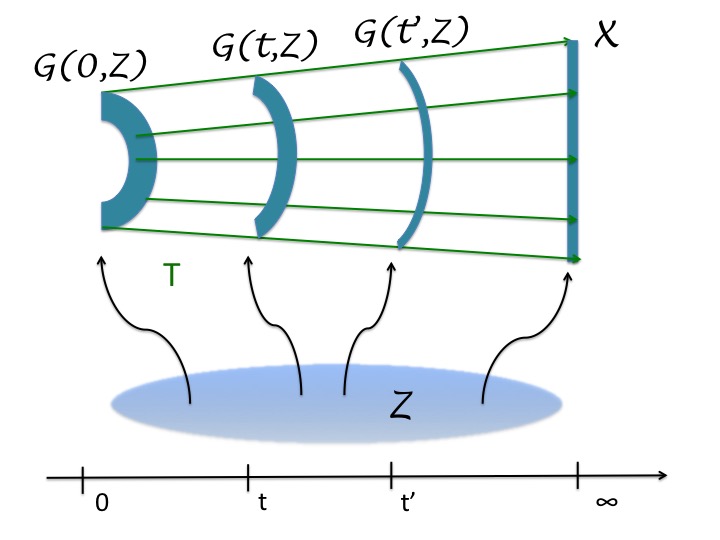}
\caption{The time evolving generated distribution minimizing its $2$-Wasserstein distance with the true distribution $X\sim P_x$. The input space is fixed and we denote it as $Z\sim P_z$. Green arrows give the shape of the optimal transport map $T$ between the initial distribution $G(0,Z)$ and the true distribution $X$. During training, $G(t,Z)$ does not follow an arbitrary path for converging toward $X$. It follows the Wasserstein-2 geodesic between $G(0,Z)$ and $X$ described by T.}
\label{Evolution_generator}
\end{center}
\end{figure}



\begin{proposition}
\normalfont
\label{unperfect_discriminator}
Given the current generated distribution $\mu_{\theta_t}$, consider a discriminator $\phi$ such that $\|\nabla \phi - \nabla \Tilde{\phi}\|_{\infty}\leqslant \epsilon$, where $\nabla \Tilde{\phi}$ is the gradient of a Kantorovitch potential $\Tilde{\phi}$ for the Kantorovitch problem \eqref{KANTOROVITCH_DUAL} between $\mu_{\theta_t}$ and $P_x$. Let $\mu_{\theta_{t+1}}$ and $\Tilde{\mu}_{\theta_{t+1}}$ be two ideally updated distributions, according to discriminators $\phi$ and $\tilde{\phi}$ respectively. Then $W_2(\mu_{\theta_{t+1}},\Tilde{\mu}_{\theta_{t+1}})\leqslant \frac{\alpha \epsilon}{\sqrt{2}}$.  
\end{proposition}
\begin{proof}
By the way ideal updates are constructed, we have $\mu_{\theta_{t+1}}=G_{t+1}\#P_z$ and $\Tilde{\mu}_{\theta_{t+1}}=\Tilde{G}_{t+1}\#P_z$ where $G_{t+1}:=(\id-\alpha \nabla \phi)\circ G_{\theta_t}$ and $\Tilde{G}_{t+1}:=(\id-\alpha \nabla \Tilde{\phi})\circ G_{\theta_t}$. Consider potentials $(f,g)\in \mathcal{A}^*(\mu_{\theta_{t+1}},\Tilde{\mu}_{\theta_{t+1}})$. Then,

\[V^*(f,g)=\int_{z\in \mathds{R}^d} f(G_{t+1}(z))dP_z(z)+ \int_{z\in \mathds{R}^d} g(\Tilde{G}_{t+1}(z))dP_z(z)\]

Unfolding the definitions, we rewrite the above as
\[V^*(f,g)=\int_{z\in \mathds{R}^d} [f((\id-\alpha \nabla \phi)( G_{\theta_t}(z)))+ \int_{z\in \mathds{R}^d} g((\id-\alpha \nabla \Tilde{\phi})(G_{\theta_t}(z)))]dP_z(z)\]
$(f,g)\in \mathcal{A}^*(\mu_{\theta_{t+1}},\Tilde{\mu}_{\theta_{t+1}})$ means that for any $(x,y)\in (\mathds{R}^m)^2$, $f(x)+g(y)\leqslant \frac{\|x-y\|_2^2}{2}$. Hence,
\[V^*(f,g)\leqslant \alpha^2\int_{z\in \mathds{R}^d} \frac{\|(\nabla \phi-\nabla \Tilde{\phi})(G_{\theta_t}(z))\|_2^2}{2} dP_z(z)\leqslant \frac{\alpha^2\epsilon^2}{2}\]
where we used that $\|\nabla \phi - \nabla \Tilde{\phi}\|_{\infty}\leqslant \epsilon$ in order to obtain the second inequality. Taking the supremum over all $(f,g)\in \mathcal{A}^*(\mu_{\theta_{t+1}},\Tilde{\mu}_{\theta_{t+1}})$ we obtain $W_2^2(\mu_{\theta_{t+1}},\Tilde{\mu}_{\theta_{t+1}})\leqslant \frac{\alpha^2\epsilon^2}{2}$.
\end{proof}

\begin{proposition}
\normalfont
\label{deviation_ideal_update}
Fix a time step $t\in \mathds{N}$. Denote $\mu_{\theta_{t+1}}$ and $\Tilde{\mu}_{\theta_{t+1}}$ the generated distributions obtained using the parametric and functional updates respectively. If $\|G_{\theta_{t+1}}-G_{t+1}\|_2\leqslant \epsilon'$ or $\|G_{\theta_{t+1}}-G_{t+1}\|_\infty\leqslant \epsilon'$ for some $ \epsilon' \geqslant 0$, then $W_2(\mu_{\theta_{t+1}},\Tilde{\mu}_{\theta_{t+1}})\leqslant \frac{\epsilon'}{\sqrt{2}}$. 

\end{proposition}
\begin{proof}
Let $(f,g) \in  \mathcal{A}^*(\mu_{\theta_{t+1}},\Tilde{\mu}_{\theta_{t+1}})$. We have
\[V^*(f,g)=\int_{z\in \mathds{R}^d} f(G_{t+1}(z))dP_z(z)+ \int_{z\in \mathds{R}^d} g(G_{\theta_{t+1}}(z))dP_z(z)\]
$(f,g)\in \mathcal{A}^*(\mu_{\theta_{t+1}},\Tilde{\mu}_{\theta_{t+1}})$ means that for any $(x,y)\in (\mathds{R}^m)^2$, $f(x)+g(y)\leqslant \frac{\|x-y\|_2^2}{2}$. Hence,
\[V^*(f,g)\leqslant \int_{z\in \mathds{R}^d} \frac{\|(G_{\theta_{t+1}}(z)-G_{t+1}(z))\|_2^2}{2} dP_z(z)\leqslant \frac{\epsilon'^2}{2}\]
where we used either $\|G_{\theta_{t+1}}-G_{t+1}\|_2\leqslant \epsilon'$ or $\|G_{\theta_{t+1}}-G_{t+1}\|_\infty\leqslant \epsilon'$ to device the last inequality. Taking the supremum over all $(f,g)\in \mathcal{A}^*(\mu_{\theta_{t+1}},\Tilde{\mu}_{\theta_{t+1}})$ we obtain $W_2^2(\mu_{\theta_{t+1}},\Tilde{\mu}_{\theta_{t+1}})\leqslant \frac{\epsilon'^2}{2}$.
\end{proof}

\begin{corollary}
\normalfont
\label{total deviation}
If
\begin{inlinelist}
    \item the discriminator $\phi$ is such that $\|\nabla \phi - \nabla \Tilde{\phi}\|_{\infty}\leqslant \epsilon$, where $\nabla \Tilde{\phi}$ is as in Proposition \ref{unperfect_discriminator}
    \item the parametric and functional updates $G_{\theta_{t+1}}$ and $G_{t+1}$ with respect to $\phi$ are such that$\|G_{\theta_{t+1}}-G_{t+1}\|_2\leqslant \epsilon'$ or $\|G_{\theta_{t+1}}-G_{t+1}\|_\infty\leqslant \epsilon'$. 
\end{inlinelist}
then $W_2(\mu_{\theta_{t+1}},\Tilde{\mu}_{\theta_{t+1}})\leqslant \frac{\alpha\epsilon+\epsilon'}{\sqrt{2}}$ where $\Tilde{\mu}_{\theta_{t+1}}$ is the ideally updated distribution with respect to the perfect discriminator $\Tilde{\phi}$. 
\end{corollary}
\begin{proof}
Denote $\mu_{\theta_{t+1},\Tilde{\phi}}$ the ideally updated distribution with respect to the discriminator $\phi$. From Proposition \ref{unperfect_discriminator}, $W_2(\mu_{\theta_{t+1},\phi},\Tilde{\mu}_{\theta_{t+1}})\leqslant \frac{\alpha \epsilon }{\sqrt{2}}$. From Proposition \ref{deviation_ideal_update}, $W_2(\mu_{\theta_{t+1}},\mu_{\theta_{t+1},\phi})\leqslant \frac{\epsilon' }{\sqrt{2}}$. We can thus conclude by triangle inequality.

\end{proof}

\subsection{Results in optimal transport theory}
\label{OT_THEORY_BACKGROUND}
We develop a bit more the materials of the background section, introducing the same notions with more details in the same order. \\

\textbf{Monge Problem} Optimal transport (OT) theory~\citep{Villani, Ambrosio2013} introduces a natural quantity to distinguish two prob١٢٠ability measures. Given two probability measures $\mu$ and $\nu$ on the euclidean space $\mathds{R}^d$, the original \textit{Monge problem} is to find a \textit{map} $T$ that ``transports" the $\mu$ distribution on $\nu$, and minimizes the cost of the transport, which is point-wise defined by a fixed cost function $c :\mathds{R}^d\times \mathds{R}^d \rightarrow \mathds{R}$. The value $c(x,y)$ can be seen as the cost for transporting a unit from $x$ to $y$. The problem is summarized by:
\begin{equation}
 \inf_{T \in \mathcal{A}_T} \int_{\mathds{R}^d} c(x,T(x)) d\mu(x)
 \label{MONGE_PROBLEM_BIS}
\end{equation}
where $\mathcal{A}_T$ is the set of all maps from $\mathds{R}^d$ to $\mathds{R}^d$ that respect the marginal $\nu$, more formally written as $T_\# \mu= \nu$. The map $T_\# \mu$ is called the \textit{push-forward measure} of $\mu$ and is defined as $T_\# \mu (A) = \mu(T^{-1}(A))$ for any Borel set $A \in \mathds{R}^d$. Figure \ref{fig:OT} provides an illustration of the shape of an optimal transport map $T$ solving~\eqref{MONGE_PROBLEM_BIS}.\\

\begin{figure}[t!]
  \centering
  \subfloat[The discrete case.]{
  \includegraphics[width=0.45\textwidth]{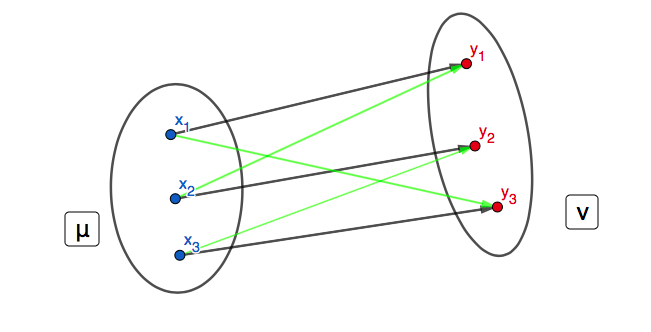}
  }
  \hfill
  \subfloat[The continuous case.]{
  \includegraphics[width=0.45\textwidth]{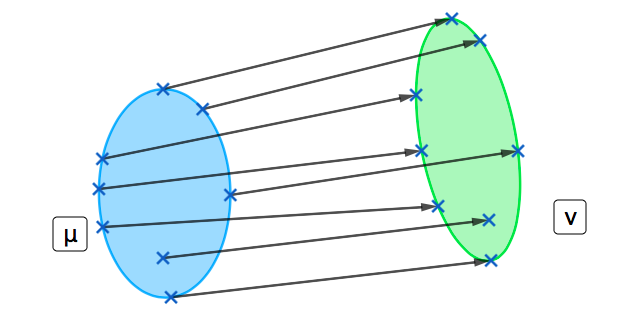}
  }
  \caption{The Monge problem. {\bf (a)} A discrete example of the Monge problem~\eqref{MONGE_PROBLEM_BIS} for distributions in $\mathds{R}^2$. The $\mu$ distribution consists in three equally weighted diracs in $x_1$, $x_2$ and $x_3$, while the $\nu$ one is represented by $y_1$, $y_2$ and $y_3$ in the same way. Black arrows denote the actual optimal transport map T. The green arrows together also define a map from $\mu$ onto $\nu$, but it is not optimal. {\bf (b)} A continuous example of the Monge problem~\eqref{MONGE_PROBLEM_BIS}. $\mu$ and $\nu$ are uniform distributions on the blue and green ellipsoids respectively. The optimal transport map T is defined for any point in the support of $\mu$, and we see how it transports some points onto $\nu$'s support with the arrows.} 
  \label{fig:OT}
 \end{figure}

\textbf{Kantorovitch relaxation} Unfortunately, problem~\eqref{MONGE_PROBLEM_BIS} often does not admit a solution as $\mathcal{A}_T$ might be empty~\footnote{For instance, let us take $\mu=\delta_0$ a dirac at $0$ and $\nu:=\frac{1}{2}(\delta_{-1}+\delta_{1})$ a weighted sum of diracs, both on real line. One can see that any map T would have to send $0$ on either $-1$ or $1$, hence the constraint $T\#\mu=\nu$ cannot hold}. To circumvent this issue, one considers the so-called \textit{Kantorovitch relaxation} of this problem:
\begin{equation}
 V_c(\mu,\nu):=\inf_{\pi \in \mathcal{A}(\mu,\nu)} \int_{ x,y \in \mathds {R}^d} c(x,y) d\pi(x,y)
 \label{KANTOROVITCH_PRIMAL}
\end{equation}
where $\mathcal{A}(\mu,\nu)$ is the set of joint distributions whose first and second marginals are equal to $\mu$ and $\nu$ respectively. The joint $\pi$ is called the \textit{transport plan} between $\mu$ and $\nu$. If $\pi$ is deterministic --- specifically, for any $x$, there is a unique $y$ such that $\pi(x,y)>0$ --- then it is also a transport map as defined in the Monge problem. It suffices to define such a map $T(x)$ by the only $y$ respecting $\pi(x,y)>0$.  However, we could instead consider $\pi$ to be a ``one-to-many" transport plan: for each $x$, there might be several $y$ such that $\pi(x,y)>0$. While Monge's development had the problem of non-existence of the transport map, the Kantorovitch relaxation $\mathcal{A}(\mu,\nu)$ is never empty. In particular, it always contains the independent joint distribution $\mu \times \nu$. In many situations, under mild assumptions on the cost function $c$ (semi-lower continuity and bounded from below), there always exists a minimizer of~\eqref{KANTOROVITCH_PRIMAL}, so that the infimum might be replaced by a minimum~\citep{Ambrosio2013}. 
We distinguish between a \emph{transport map} $T$ achieving the minimum in~\eqref{MONGE_PROBLEM_BIS} and a \emph{transport plan} $\pi$ minimizing~\eqref{KANTOROVITCH_PRIMAL}. Note that for most applications, we are more interested in obtaining an approximation of the optimal transport map than a transport plan.\\

\textbf{Dual of Kantorovitch} One could also work with the dual of problem~\eqref{KANTOROVITCH_PRIMAL}, which is proven to lead to the same value $V_c(\mu,\nu)$. We write this dual as:
\begin{equation}
\begin{gathered}
\underset{\phi,\psi \in \mathcal{A}^* }{\text{sup}}
 \int_{ x \in \mathcal{X}} \phi(x) d\mu(x) + \int_{ y \in \mathcal{Y}} \psi(y) d\nu(y) \\
\mathcal{A}^*:= 
\{(\phi,\psi): \mathds{R}^d\rightarrow \mathds{R} \mid
\forall x,y  \in \mathcal{X}\times \mathcal{Y}, \\
 \phi(x)+\psi(y) \leqslant c(x,y)\}
\end{gathered}
\label{KANTOROVITCH_DUAL_BIS}
\end{equation}
As in the main paper, we will also denote as $V^*(\phi,\psi)$ the value $\int_{ x \in \mathcal{X}} \phi(x) d\mu(x) + \int_{ y \in \mathcal{Y}} \psi(y) d\nu(y)$. A coordinate of a pair ($\phi,\psi$) maximizing $V^*$ is called a \textit{Kantorovitch potential}. We refer to the constraint in the definition of~\eqref{KANTOROVITCH_DUAL_BIS} which forces the sum of the potentials $\phi$ and $\psi$ to be upper-bounded by the cost as the \emph{c-inequality constraint} or simply the \emph{inequality constraint}.\\

\textbf{Wasserstein distances} When the cost $c$ is a distance, $V_c(\mu, \nu)$ exactly matches the famous $W_1(\mu,\nu)$, which is called the first Wasserstein distance between $\mu$ and $\nu$. When the cost $c$ is a distance to a power of some positive integer $p$, $V_c(\mu, \nu)^{1/p}$ is denoted $W_p(\mu, \nu)$ and is commonly called the $p^{\mathrm{th}}$ Wasserstein distance. An important result is that these $W_p$ distances are actually respecting the axioms of a distance over (absolutely continuous) distributions~\citep{Villani}, although one should pay careful attention to the fact that equality between measures is meant in an "almost everywhere" sense. For the sake of simplicity, we will abusively extend the term ``Wasserstein distance" to any $V_c$, for any cost function c, even though not all choices of $c$ lead to $V_c$ being a metric over probability distributions. \\

\textbf{Computing Wasserstein distances} The discriminator of the Wasserstein GAN (WGAN)~\citep{arjovsky2017wasserstein} computes the dual formulation~\eqref{KANTOROVITCH_DUAL_BIS}, with $\nu$ being the ``true" distribution $P_x$ (defined by training examples) and $\mu$ being the ``fake" distribution $G_\# P_z$ of the random variable $G(z)$ (defined by samples $z$ drawn from the input distribution $P_z$ and then fed to the generator). In practice, the inequality constraint on potentials $\phi,\psi$ may be enforced by the addition of a constraint-violation penalty term to the objective~\citep{gulrajani2017improved}. Although introduced for other reasons, \mycitet{cuturi2013sinkhorn} proposed an efficient penalized version of~\eqref{KANTOROVITCH_PRIMAL} whose resulting optimization problem is called the entropic-regularized optimal transport problem. We shall only provide its dual form: 
\begin{equation}
\sup_{\phi,\psi} V^*(\phi,\psi) + \lambda_{\text{ineq}} \mathcal{L}_{\text{ineq}}(\phi,\psi)
\label{INEQ_REGULARIZED_DUAL_BIS}
\end{equation}
where $\mathcal{L}_{\text{ineq}}$ penalizes $(\phi,\psi)$ when they violate the inequality constraint:
\begin{align}
    \mathcal{L}_{\text{ineq}}(\phi,\psi):= - \mathds{E}_{(x,y) \sim \mu \times \nu} \left[\exp\left({\frac{\phi(x)+\psi(y)-c(x,y)}{\lambda_{\text{ineq}}}}\right)\right]
\end{align} 
Other penalties were theoretically explored by~\mycitet{Blondel2018}, including the widely used $L_2$-penalty~\citep{Seguy2018LSOT}:
\[ \mathcal{L}_{\text{ineq}}(\phi,\psi):=- \mathds{E}_{(x,y) \sim \mu \times \nu} \left[\left(\phi(x)+\psi(y)-c(x,y)\right)_{+}^2 \right]    \]
where $(.)_+:= \max(0,.)$. Objective~\eqref{INEQ_REGULARIZED_DUAL_BIS} was shown to asymptotically recover the value of~\eqref{KANTOROVITCH_DUAL_BIS} when $\lambda_{\text{ineq}} \rightarrow 0$. 

Due to its suitability for training parametric models such as neural network, we exploit this optimization objective to compute Wasserstein distances $W_p$. As it is provable~\citep{Ambrosio2013} that the `differential' of $W_p(.,\nu)$ with respect to its first variable is $-\phi$ (where $\phi$ is a Kantorovitch potential), we may conclude that a GAN framework where the discriminator accurately computes~\eqref{KANTOROVITCH_DUAL_BIS} through~\eqref{INEQ_REGULARIZED_DUAL_BIS} provides a means of training a generative model.\\

\textbf{The link between the Kantorovitch potentials and the Monge map} We pursue this section by exploring the relationship between optimal potentials and optimal transport plans and maps. The goal is to recover an approximation of the solution of~\eqref{MONGE_PROBLEM_BIS}, namely an optimal transport map, given optimal solutions of~\eqref{KANTOROVITCH_DUAL_BIS} (i.e. Kantorovitch potentials). There is the following well-known result about the existence and uniqueness of the optimal map and its relationship to Kantorovitch potentials~\cite{brenier1991polar} is:

\begin{proposition}\label{prop:uniqueRelation}
\normalfont
Let $\mathcal{X}=\mathcal{Y}=\mathds{R}^m$ for some integer $m$. Fix $p\geqslant 2$ and the cost $c(x,y):= \frac{\|x-y\|^p}{p}$. Further assume $\mu,\nu$ to be absolutely continuous with respect to the Lebesgue measure on $\mathds{R}^m$. Then there is one unique optimal transport plan $\pi$ solving~\eqref{KANTOROVITCH_PRIMAL}. It is deterministic, that is it corresponds to an optimal transport map $T$ solving~\eqref{MONGE_PROBLEM_BIS}: $\pi=(\id \times T)$. Also, the Kantorovitch potentials $\phi,\psi$ are unique up to a translation and the following relation holds:
\begin{equation}
\forall x \in \text{supp}(\mu),\;\; T(x)=x- \|\nabla \phi(x)\|^{\frac{1}{p-1}-1}\nabla \phi(x)
\label{PRIMAL_DUAL_GENERAL_CORRESPONDANCE}
\end{equation}
\end{proposition}

Finally, we recall a useful result for the converse question: given a function $T$ from $R^d$ to itself and a probability $\mu$, is $T$ a Monge map between $\mu$ and $T\#\mu$? In the case of the second Wasserstein distance, Brenier's polarization theorem~\citep{brenier1991polar} says this is true if $T$ is the gradient of a strictly convex function. If $T$ may be written $x-\nabla \phi(x)$ for some real function $\phi$ one can rephrase this condition as $\frac{\|x\|_2^2}{2}-\phi(x)$ is strictly convex. Such a $\phi$ is said to be c-concave. In particular, any Kantorovitch potential $\phi$ is c-concave.\\

\textbf{Geodesics in the space of probability measures}
Recall that $W_2$ is a metric over absolutely continuous distributions. Studying properties of this space starts with the analysis of its geodesics. That is, for any given distributions $\mu$ and  $\nu$, we look for constant speed paths described by $\mu_t$ such that $\mu_0=\mu$, $\mu_1=\nu$ and
\[\forall 0\leqslant s \leqslant t\leqslant1, W_2(\mu_s,\mu_t)=(t-s)W_2(\mu,\nu) \]
Between a fixed pair $(\mu,\nu)$, there exists a unique constant-speed $W_2$ geodesic~\citep{Villani}. Given a Monge map $T$ between $\mu$ and $\nu$, the only constant speed geodesic is given by $\mu_t:=T_t\#\mu$ where $T_t=(1-t)\id+tT$. It is a remarkable fact, compared with the case of the first Wasserstein distance, where the geodesics are eventually in infinite number. This is one obstruction to conduct our analysis in the $W_1$ case, and this is one reason we appeal to $W_2$ instead of $W_1$.

By~\eqref{PRIMAL_DUAL_GENERAL_CORRESPONDANCE}, we can write $T(x)=x-\nabla \phi(x)$ for any Kantorovitch potential $\phi$ solving~\eqref{KANTOROVITCH_DUAL_BIS}. Hence $T_t(x)=x-t\nabla\phi(x)$. In particular, this guarantees that $t\phi(x)$ is a Kantorovitch potential solving~\eqref{KANTOROVITCH_DUAL_BIS} for $\mu$ and $\mu_t=T_t\#\mu$.

\subsection{Can Wasserstein GANs benefit from a similar analysis?}
\label{sec:W1_case}
Our model is similar to GANs relying on the $1$-Wasserstein metric, and one might wonder why we do not use 
 e.g WGAN, WGAN-GP or WGAN-LP as they achieve state-of-the art performance in generative modelling. For the purpose of learning an optimal map, this section is meant to explain why we cannot theoretically rely on such models, at least by trying to adapt our analysis to the $W_1$ case. Experimentally, however, as WGAN and its extensions rely on a similar objective as W2GAN, they seem to be approximately following an optimal map -- at least in some cases. Let us enumerate what parts of the previous analysis do not apply to the case of $W_1$:
\begin{itemize}
    \item The local analysis could actually be applied in a similar way. Recall that local direction given to the generated distribution by the Wasserstein GAN discriminator is $-\nabla f(.)$ where $f$ is a (non unique) Kantorovitch potential in the $W_1$ case. It is true then that there exists an associated Monge map $T$ from the current distribution toward the true distribution. The relationship between the map and the potential is $\frac{T(x)}{\|T(x)\|}=-\nabla f(x)$~\citep{brenier1991polar}. Hence, we can get the direction of an OT map thanks to the knowledge of the dual optimal variable, but not its norm. Thus, locally, the generator is updated toward the true distribution, and in the direction of the optimal map, but the magnitude of the update step is impossible to obtain. 
    \item Going from a local to global analysis is much more difficult in the case of $W_1$. Recall that in the case of $W_2$, the main ingredient of the discussion was the uniqueness of geodesics and their nice analytic description. In the case of $W_1$, there are eventually infinitely many geodesics. A great description of this fact is available in~\citep{bottou2018geometrical}. Hence, while locally the discriminator may be one among many Kantorovitch potentials, it is hard to decide whether globally the generator follows a geodesic. Although this seems to be true as every local direction is happening on geodesics, it would remain impossible to decide which geodesic the generator is following. It would be of interest to describe different kind of $W_1$-geodesics and understand the ones that Wasserstein GANs' generators are prone to follow.\footnote{For the same reason it is also hard to solve the gradient flow of the $W_1$ metric as we do in the Appendix~\ref{subsection_dynamic_gradient_flow} 
    in the case of $W_2$.} 
\end{itemize}
In conclusion, while we might have a locally nice description of GANs relying on $W_1$, it remains challenging to obtain a global description. Nonetheless, low dimensional experiments still suggest that their generators might recover a Monge map at the end of training.

\subsection{Proofs of Appendix' propositions}
\begin{proof}[proposition \ref{prop:uniqueRelation}]
This result is a particular case of a well-known correspondence between Kantorovitch potentials and optimal transport map. In fact, when the cost $c$ is such that $c(x,y)=h(x-y)$ and $h$ is strictly convex, one has the existence and uniqueness of an optimal Monge map $T$ solving~\eqref{MONGE_PROBLEM_BIS} and a specific relationship with the Kantorovith potential $\phi,\psi$ solving~\eqref{KANTOROVITCH_DUAL_BIS}~\mycitet{Villani,santambrogio2017euclidean}:
\begin{center}
$T(x)=x- (\nabla h)^{-1}(\nabla \phi(x)) $
\end{center}
Fix $p \geqslant 2$ and consider the case of the p-Wasserstein distance. Then $c(x,y)=h(x-y)$ where $h(x):= \frac{1}{p} \|x\|^p$. A norm is always convex by triangle inequality, and any $x \rightarrow x^p$ is strictly convex and increasing on $\mathbf{R}_+$, so the previous result provides the uniqueness of the optimal transport map $T$. It only remains to invert the gradient of h. A quick calculation gives:
\[ \forall y \in Im\nabla h, \\ \nabla h ^{-1}(y)= \|y\|^{\frac{1}{p-1}-1}y\]
as we supposed the norm to be the euclidean one. Plugging this into the first expression, we obtain the desired result.
\end{proof} 

\begin{proof}[proposition \ref{prop:dynamic}]

Proving the first part of proposition 2 is exactly similar as the way it is done for the Wasserstein-1 case in~\mycitet{arjovsky2017wasserstein}. We recall that the main ingredients of the proof is first to convey an envelop theorem to obtain that 
\[\nabla_\theta W_2^2(\mu_{\theta_t},P_x)=\nabla_\theta \mathds{E}_{z\sim P_z}[\phi(G(\theta,z))]\]
and second to use a dominated convergence argument to invert the expectation and the gradient operator in the right-hand side of the above.\\
For the second part of the proof, we rewrite the gradient ascent equation~\eqref{UPDATE_THETA} for the generator parameter $\theta$:
\begin{equation}
    \begin{aligned}
    \theta_{t+1}&=\theta_t -\alpha \nabla_\theta W_2^2(\mu_{\theta_t},P_x)\\
    &= \theta_t -\alpha  \mathds{E}_{z_ \sim P_z}[\nabla_\theta \phi(G(\theta_t,z))]\\
    &=\theta_t -\alpha  \mathds{E}_{z \sim P_z}[J_{\phi(G(\theta_t,z))}^T(\theta_t)] \\
    &=\theta_t -\alpha  \mathds{E}_{z\sim P_z}[( (\nabla \phi(G(\theta_t,z)))^T J_{G(.,z)}(\theta_t) )^T]\\
    &=\theta_t -\alpha  \mathds{E}_{z\sim P_z}[(  J_{G(.,z)}(\theta_t)^T \nabla \phi(G(\theta_t,z)) ]
    \end{aligned}
\end{equation}

Therefore, for a fix input variable $z_0$,
\begin{align}
G(\theta_{t+1},z_0)= G(\theta_t -\alpha  \mathds{E}_{z\sim P_z}[  J_{G(.,z)}(\theta_t)^T \nabla \phi(G(\theta_t,z)) ],z_0)
\end{align}
Then by first order Taylor expansion: 
\begin{align}
G(\theta_{t+1},z_0)= G(\theta_t,z_0) -\alpha J_{G(.,z_0)}(\theta_t) \mathds{E}_{z\sim P_z}[  J_{G(.,z)}(\theta_t)^T \nabla \phi(G(\theta_t,z) ]  + o(\alpha)
\end{align}

We conclude using the hypothesis that $\phi$ is maximizing the dual of Kantorovitch and is related to an optimal map T through~\eqref{PRIMAL_DUAL_GENERAL_CORRESPONDANCE}.
\end{proof}

\subsection{Extensions to the W2GAN training objective}
\label{sec:objective_details}

In the following, we discuss two possible extensions to the discriminator's training objective inspired by optimal transport theory. We consider the two random variables $X\sim \mu$ and $Y\sim \nu$ in $\mathds{R}^m$ to be discriminated by $\phi,\psi$. In our applications, $X$ is the generated distribution $G(z)$ and $Y$ is the true distribution $P_x$. In both cases the discriminator objective is divided into two parts. First, a main objective $\mathcal{L}_{\text{OT}}(\phi,\psi,X,Y):=\mathds{E}(\phi(X)+\psi(Y))$ which corresponds to the dual of the Kantorovitch problem. Second, the inequality constraint $\mathcal{L}_{\text{ineq}}(\phi,\psi,X,Y):=-\lambda_{\text{ineq}} \mathds{E} [(\phi(X)+\psi(Y)-\frac{\|X-Y\|_2^2}{2})_{+}^2 ]$. One could prefer an other choice of regularization such as the entropy penalty. In any case the overall objective for the discriminator is:
\begin{equation*}
    \sup_{\phi,\psi} \mathcal{L}_{\text{OT}}(\phi,\psi,X,Y) + \mathcal{L}_{\text{ineq}}(\phi,\psi,X,Y)
\end{equation*}

The first idea is to take advantage of relation~\eqref{PRIMAL_DUAL_GENERAL_CORRESPONDANCE}, which we know to be true between an optimal potential $\phi$ and an optimal map $T$. In fact, optimal transport theory asserts that the inequality constraint should be exactly saturated where there is some transport~\citep{Villani}:
\begin{theorem}
Consider any lower-semicontinuous cost function $c$ and a optimal transport plan for~\eqref{KANTOROVITCH_PRIMAL}, and Kantorovitch potentials $\phi,\psi$ for~\eqref{KANTOROVITCH_DUAL_BIS}. Then,
\begin{equation}
 \forall x,y,  (x,y)\in supp(\pi)  \\ \Longrightarrow  \phi(x) + \psi(y) = c(x,y)
 \label{GENERAL_RELATIONSHIP_PI_COST_EQ}
\end{equation}
\end{theorem}

In our case, that is when the cost is a $p^{th}$ power of the euclidean distance, we know from Proposition 1 than an optimal transport plan is actually an optimal transport map $T$, and we dispose of a relationship with the corresponding Kantorovitch potential $\phi$. Hence an immediate consequence of the above for the case of the square of the euclidean distance is:

\begin{corollary}
For the Kantorovitch problem~\eqref{MONGE_PROBLEM_BIS} related to $c(x,y):=\frac{\|x-y\|_2^2}{2}$, that is the computation of the second Wasserstein distance, given Kantorovitch potentials $\phi,\psi$:
\begin{equation}
 \forall x \in supp(\mu),  \phi(x) + \psi(x-\nabla \phi(x)) =\frac{\|\nabla \phi(x)\|_2^2}{2}
 \label{EQ_COST_PHI_PSI}
\end{equation}
\end{corollary}

Thus we suggest enforcing our discriminator to abide by relation~\eqref{EQ_COST_PHI_PSI}. Notice that the previous corollary admits an exact symmetric relationship involving the gradient of $\psi$. This is done by adding the following penalty during training:
\begin{align}
\mathcal{L}_{\text{eq}}(\phi,\psi,X,Y):=
-\lambda_{\text{eq}}& \left[ \left(\phi(X) + \psi(X-\nabla \phi(X)) -\frac{\|\nabla \phi(X)\|_2^2}{2}\right)^2 \right. \nonumber \\
& \left. + \left(\phi(Y-\nabla \psi(Y)) + \psi(Y) -\frac{\|\nabla \psi(Y)\|_2^2}{2}\right)^2 \right]
\label{GAN_EQ_PENALTY}
\end{align}
We call it the c-equality penalty term as it tries to enforce the dual functions to saturate to c-inequality constraint at the right locations. Hence the overall objective of the discriminator is:
\begin{equation}
    \sup_{\phi,\psi} \mathcal{L}_{\text{OT}}(\phi,\psi,X,Y) + \mathcal{L}_{\text{ineq}}(\phi,\psi,X,Y) + \mathcal{L}_{\text{eq}}(\phi,\psi,X,Y)
\label{DISCRIMINATOR_LOSS_EQ_INEQ}
\end{equation}

\begin{remark}
\normalfont It is interesting to bridge this objective with the one used in GAN relying on the first Wasserstein distance. In WGAN-GP, the discriminator is asked to have gradient exactly equal to $1$. Translating that into the optimal transport theory, it is tantamount to have potentials $\phi,\psi$ saturating the c-inequality. That is, this gradient penalty term is similar to our $\mathcal{L}_{\text{eq}}$. Importantly though,~\mycitet{petzka2018on} raised the fact that it is not a valid practice according to theory to ask the gradient norm of the discriminator to be $1$ everywhere. Instead, WGAN-LP is a model where the gradient's norm is enforced to be less than 1, which is then translated into optimal transport perspective as potentials $\phi,\psi$ respecting the c-inequality. Their penalty term is thus similar to our $\mathcal{L}_{\text{ineq}}$. Our model takes advantage of the two forms. The difference is that our $\mathcal{L}_{\text{eq}}$ is completely justified -at least because it enforces a relationship which is true at optimum- thanks to the theory existing in the $W_2$ case.
\end{remark}

A second modification concerns the way we encode the penalty enforcing the inequality constraint to be respected by $\phi,\psi$. In fact, we modify $\mathcal{L}_{\text{ineq}}$ in order to bring it closer to the theory. Looking at the definition of $\mathcal{A}^*(\mu,\nu)$, the set of constraint of the dual of Kantorovitch problem~\eqref{KANTOROVITCH_DUAL_BIS}, we see that the c-inequality constraint should be respected point wise everywhere in the ambient space where the distributions are defined. Instead, the entropy regularization or the $L^2$ one only enforce it for pairs $x,y$ that live in the supports of the two distributions. We want to reduce this bias by somehow enforcing the inequality on a broader set. It is sufficient that the potentials respect the inequality constraint point wise on a convex compact set $\Omega$ containing the support of the two distributions. Hence we suggest enforcing it on the convex envelop of the two supports: $\Omega:=$Conv(supp($\mu$) $\bigcup$ supp($\nu$)). To do so, we define two i.i.d random variables $\tilde{X}$ and $\tilde{Y}$ which follow the same law as $\epsilon X + (1- \epsilon)Y$ where $\epsilon \sim \mathcal{U}([0,1])$. Hence the overall objective for $\phi,\psi$ is:
\begin{align}
    \sup_{\phi,\psi} \mathcal{L}_{\text{OT}}(\phi,\psi,X,Y) + \mathcal{L}_{\text{ineq}}(\phi,\psi,\tilde{X},\tilde{Y}) +
 \mathcal{L}_{\text{eq}}(\phi,\psi,X,Y)
\label{DISCRIMINATOR_LOSS_EQ_INEQ_INTERP}
\end{align}

\begin{remark}
\normalfont This interpolation idea has already been used in GANs relying on the first Wasserstein distance, such as WGAN-GP and WGAN-LP in their gradient penalty. This practice was more motivated by better results. Here we provided a theoretical argument in favor of such practice. On the other hand, models trying to broaden GANs to higher order Wasserstein distances and/or to compute optimal transport map~\citep{Sanjabi2018swgan,Seguy2018LSOT,Salimans2017otgan} in a similar manner only enforced the constraint on the support of the distributions.
\end{remark}

\subsection{An alternative parameterization for $\phi$ and $\psi$}
\label{sec:parametrization_details}
An intuitive way of parameterizing $\phi$ and $\psi$ is to simply replace $\phi$ and $\psi$ with two neural networks of the same architecture. A potential downside we found practically is that this parameterization tends to be unstable. An alternative reparameterization is to replace $\psi(Y)$ with $-\phi(Y) + \epsilon(Y)$ where both $\phi$ and $\epsilon$ are neural networks. This reparameterization and the property that 
$$\forall y \in \mathds{R}^m, \; \phi(y) + \psi(y) \leq \frac{c(y,y)}{2}= 0 \; \Longrightarrow \; \epsilon(y)\leqslant 0,$$ which yields the additional regularizer:
\begin{equation}
    \mathcal{L}_{\epsilon}(\phi, \psi, X, Y) = -\lambda_{\epsilon}\bigg[\epsilon(Y)_+^2\bigg]\label{addition}.
\end{equation}
We found this parameterization to be particularly useful for high-dimensional data settings.

\subsection{Training algorithm with additional regularizers}
\label{sec:algorithm_details}
In algorithm \ref{alg:example2}, one can choose to use the equality constraint as an additional constraint for the discriminator. One can also use the interpolation method, or any method of sampling in Appendix \ref{sec:objective_details}. In the pseudo-code presented, we combine all methods but discarding some of the additional losses $\mathcal{L}_{eq}$ or $\mathcal{L}_{\epsilon}$ amounts to setting the hyper parameters $\lambda_{eq}$ or $\lambda_\epsilon$ to $0$.

\begin{algorithm}[t]
  \caption{W2GAN with $D = (\phi, \psi)$ and $L^2$-regularization.}
  \label{alg:example2}
\begin{algorithmic}
  \REQUIRE Hyperparameters $\lambda_\mathrm{eq}$, $\lambda_\mathrm{ineq}$, $\lambda_{\epsilon}$, $n_\mathrm{critic}$, $B$, $p$.
  \REQUIRE Initial parameters $w_0$ of $\phi$, $v_0$ of $\psi$ and $\theta_0$ of $G$.
  \WHILE{$\theta$ has not converged}
  \STATE {Initialize generator and discriminator losses $\mathcal{L}_D$,$\mathcal{L}_G$ to 0.}
  \FOR{$t=1, ..., n_{\mathrm{critic}}$}
  
  \STATE {Initialize losses $\mathcal{L}_{\text{OT}}$, $\mathcal{L}_{\text{eq}}$ and $\mathcal{L}_{\text{ineq}}$ to 0.}
  \FOR{$i=1, ..., B$}
  \STATE Sample real data $x,x' \sim \mathbf{P}_x$, input variable $y,y' \sim \mathbf{P}_{G(z)}$ and $\epsilon_1, \epsilon_2 \sim \mathcal{U}([0,1])$.
  \STATE $\mathcal{L}_{\text{OT}} \leftarrow \mathcal{L}_{\text{OT}} - \frac{1}{B}(\phi(x) + \psi(y))$
  \STATE $\mathcal{L}_{\text{eq}} \leftarrow \mathcal{L}_{\text{eq}} + \frac{\lambda_{\text{eq}}}{B}([\phi(x) + \psi(x-\nabla \phi(x))-\frac{\|\nabla \phi(x)\|_2^2}{2}]^2 + [\phi(y-\nabla \psi(y)) + \psi(y)-\frac{\|\nabla \psi(y)\|_2^2}{2}]^2)$
  \STATE $\tilde{x} \leftarrow \epsilon_1 x+ (1-\epsilon_1) y$, $\tilde{y} \leftarrow \epsilon_2 x'+ (1-\epsilon_2) y'$
  \STATE \STATE $\mathcal{L}_{\text{ineq}} \leftarrow \mathcal{L}_{\text{ineq}} + \frac{\lambda_{\text{ineq}}}{B}(\phi(\tilde{x}) + \psi(\tilde{y})- \frac{\|\tilde{x}-\tilde{y}\|_2^2}{2})_+^2$
  \STATE \STATE $\mathcal{L}_{\epsilon}\leftarrow \mathcal{L}_{\epsilon} + \frac{ \lambda_{\epsilon}}{B}\epsilon(y)_+^2$
  \ENDFOR
  \STATE $\mathcal{L}_D \leftarrow \mathcal{L}_{\text{OT}} +\mathcal{L}_{\text{ineq}} +\mathcal{L}_{\text{eq}}+\mathcal{L}_{\epsilon}$
  \STATE Update the parameter $w$ of $(\phi,\psi)$ with respect to $\mathcal{L}_D$ via SGD.
  \ENDFOR
  \STATE $\mathcal{L}_G \leftarrow \mathcal{L}_{\text{OT}}$.
  \STATE Update the parameter $\theta$ of $G$ with respect to $\mathcal{L}_G$ via SGD.
  \ENDWHILE
\end{algorithmic}
\end{algorithm}

\subsection{Implementation Details and Architecture}\label{sec:exp_details}

Below we describe the implementation details of our experiments.
\subsubsection{2D Setting}
For the 2D synthetic data experiments, the learning rate used for Barycentric-OT is $0.005$ (although we did not notice that the learning rate influenced the solution quality significantly). For the GAN experiments, learning rates were chosen from the set $\{0.00001, 0.0005, 0.00005, 0.00001\}$. For the Wasserstein-based GANs, the number of discriminator updates per generator update is chosen from the set $\{5, 10, 20\}$. This is set to $1$ in the Jensen-Shannon based GAN by default. $\lambda_{\text{gp}}$ for both WGAN-LP and WGAN-LP is set to $10$, following their conventions. For W2-OT and W2GAN, we set $\lambda_{\text{eq}} = \lambda_{\text{ineq}} = 200$.
To enforce that $G_{\theta_0}(z) = z$, we parameterize $G$ by $G(z) = H(z) + z$, where $H(z)$ is initialized to be close to $0$. $H$ is parameterized by 4 fully connected hidden layers of size 128, with ReLU activations and batch norm in between the layers, and 1 fully connected final layer. $\phi$ and $\psi$ are each parameterized by 2 fully connected layers with ReLU activations in between, and 1 fully connected final layer.

\subsubsection{Multivariate Gaussian to MNIST Setting}
In this experiment, MNIST images are kept at their original size of $28 \times 28$ and pixel values are re-scaled to be in $[-1, 1]$.
For Barycentric-OT, we use the same architecture as~\citep{Seguy2018LSOT}. We searched over $\lambda \in \{0.01, 0.05, 0.1, 0.5, 1, 2, 5\}$, and use the ADAM optimizer with learning rate$ = 0.0002$ and $\beta_1 = 0.5$, $\beta_2 = 0.999$ for both the dual variables and the mapping. We run the experiment with a batch size of $64$ for $200,000$ iterations for each phase.

For W2GAN, to enforce that $G_{\theta_0}(z)$ is close to identity, and that $G(z) \in [-1, 1]$ we reparameterize $G$ by $G(z) = 2\cdot H(z) + z_{\text{clip}}$ where $z_{\text{clip}}$ is $z$ clipped to be in $[-1, 1]$ and tanh is used as the final activation layer of $H$. We use the same architecture as barycentric-OT with the addition of BatchNorm in between the layers of the generator. Specifically, the architecture for our model is FC($28^*28 \rightarrow 1024$)-BN-RELU-FC($1024 \rightarrow 1024$)-BN-RELU-FC($1024 \rightarrow 28^*28$)-Tanh for $H$ in the generator, and FC($28^*28 \rightarrow 1024$)-RELU-FC($1024 \rightarrow 1024$)-RELU-FC($1024 \rightarrow 1$) for both $\phi$ and $\epsilon$ in the discriminator. We note that in the high dimensional setting, better training stability and image quality is achieved by using both $\mathcal{L}_{\text{eq}}$ and $\mathcal{L}_{\epsilon}$ which complement $\mathcal{L}_{\text{ineq}}$ in enforcing the constraint.  We set $\lambda_{\text{ineq}} = \lambda_{\text{eq}} = \lambda_{\epsilon} = 10$ and use the ADAM optimizer with learning rate $= 0.0001$  and $\beta_1 = 0.5$, $\beta_2 = 0.999$ for both the generator and the discriminator. We ran the experiment for $100,000$ iterations with a batch size of $64$. 

\subsubsection{Unsupervised Domain Adaptation Setting}
 The USPS dataset consists of $16 \times 16$ grayscale images of digits, with significantly less training and testing data (7291 train and 2007 test images). The MNIST digits are rescaled to $16 \times 16$ to match the USPS digits, and the grayscale pixels in both datasets are scaled to be in $[-1, 1]$. 
For this set of experiments, we use the same architecture as~\citep{Seguy2018LSOT} for Barycentric-OT, and choose the best model between using entropy regularization and L2 regularization and $\lambda \in \{0.01, 0.05, 0.1, 0.5, 1, 2, 5\}$. We use the ADAM optimizer with learning rate$ = 0.0002$ and $\beta_1 = 0.5$, $\beta_2 = 0.999$ for both the dual variables and the mapping. We ran the experiment for $20,000$ iterations for each phase with a batch size of $1024$. 

For W2GAN, to ensure that $G_{\theta_0}(z)$ is identity and that $G(z) \in [-1, 1]$, we reparameterize $G$ by $G(z) = 2\cdot H(z) + z$ where the last activation layer of $G$ is tanh. we use similar architecture as the previous experiment for our model. Specifically, the architecture is FC($16^*16 \rightarrow 200$)-BN-RELU-FC($200 \rightarrow 500$)-BN-RELU-FC($500 \rightarrow 16^*16$)-Tanh for $H$ in the generator, and FC($16^*16 \rightarrow 200$)-RELU-FC($200 \rightarrow 500$)-RELU-FC($500 \rightarrow 1$) for both $\phi$ and $\epsilon$ in the discriminator. We use the same optimizers and hyperparameters as experiment above. 
are also based on~\citep{gulrajani2017improved}.

\subsection{Continuous analysis of the generator's evolution}
\label{subsection_dynamic_gradient_flow}
 We here strengthen the argument that during training, the generated distribution $\mu_\theta$ is following the "line" defined the optimal transport map between its initialization $\mu_{\theta_0}$ and the target distribution $P_x$. As in the previous discussion, we assume the discriminator $(\phi,\psi)$ to compute exactly the squared second Wasserstein distance $W_2^2(\mu_\theta,P_x)$, i.e. we suppose it is trained infinitely many times at each update of the generator, and we forget about the bias induced by the c-inequality constraint being encoded in the objective as a penalty (informally, we assume $\lambda_{\text{ineq}}=0$). We again look at the evolution of the generated distribution in the case of ideal updates in the space of probability measures. The difference in this case is that we consider $\alpha\rightarrow 0$, thus writing $G(t,Z)$ as a time-dependent random variable where G is in the space of $L_2$ functions and $t$ is a fictive time variable. In this case we may equivalently work on the corresponding generated probability distribution of interest, which we denote as $\mu_t:=\mu_{\theta_t}$. Then our generated distribution evolves according to the gradient flow
\begin{equation}
\dot{\mu_t}= - \nabla W_2^2(\mu_t,P_x).
\label{GRADIENT_FLOW_GENERATED}
\end{equation}
One would need to introduce the definition of gradient flow in the space of probability measures, in particular the notion of real functions' gradient with respect to probability measures and velocities of time-dependent measures in order to fully express the meaning of the above. We refer to~\mycitet{ambrosio2008gradient} for a comprehensive overview. 
We can refine~\eqref{GRADIENT_FLOW_GENERATED} thanks to~\mycitet{ambrosio2008gradient} and obtain:
\begin{equation}
\dot{\mu_t}= T_t-\id
\label{GRADIENT_FLOW_GENERATED_OT}
\end{equation}
where $T_t$ is the unique optimal transport map solving~\eqref{MONGE_PROBLEM} between $\mu_t$ and $P_x$. According to the gradient flow \eqref{GRADIENT_FLOW_GENERATED_OT}, locally, we recover that the generated distribution evolves towards $P_x$ by following the Wasserstein-2 geodesic. In fact, this is a global behaviour, from \mycitet{ambrosio2008gradient}:
\begin{theorem}
Denote $T$ the optimal transport map between $\mu_0$ and $P_x$. Then we have that the gradient flow solving~\eqref{GRADIENT_FLOW_GENERATED} is uniquely determined:
\begin{equation}
\mu_t= [e^{-t}\id + (1-e^{-t})T]\#\mu_0
\label{EVOLUTION_GENERATED}
\end{equation}
\end{theorem}
Hence a consequence (also from \mycitet{ambrosio2008gradient}) is that the generated distribution evolves exponentially fast towards $P_x$:
\begin{corollary}
\begin{equation}
     \forall t\geqslant 0, W_2^2(\mu_t,P_X)=e^{-2t}W_2^2(\mu_0,P_x)
    \label{FAST_CONVERGENCE}
\end{equation}
\end{corollary}
The evolution in~\eqref{EVOLUTION_GENERATED} suggests that the generated distribution follows the Wasserstein-2 geodesic between $\mu_0$ and $P_x$. That means the training dynamic of the generator "draws" the optimal transport map between $\mu_0$ and $P_x$. At the end of training, the generator $G(\infty,.)$ provides a certain optimal transport map. For each $z$, the 'arrows' joining $G(0,z)$ and $G(\infty,z)$ together constitute the optimal map. Figure~\ref{Evolution_generator} helps visualizing this analysis.

\end{document}